\documentclass[11pt]{article}

\usepackage[margin=1in]{geometry}

\usepackage{amsthm,amsmath,amssymb,enumitem}
\usepackage{authblk,setspace,booktabs,graphicx}
\usepackage[numbers,sort&compress]{natbib}

\theoremstyle{plain}
\newtheorem{theorem}{Theorem}[section]
\newtheorem{remark}[theorem]{Remark}
\newtheorem{lemma}[theorem]{Lemma}
\newtheorem{claim}[theorem]{Claim}
\newtheorem*{lemma*}{Lemma}
\newtheorem*{theorem*}{Theorem}
\newtheorem*{claim*}{Claim}
\newtheorem*{remark*}{Remark}
\newcommand{\R}{\mathbb{R}}
\newcommand{\E}{\mathbb{E}}

\newcommand{\V}{\mathcal{V}}

\newcommand{\tS}{\tilde{S}}

\newcommand{\e}{\varepsilon}

\newcommand{\al}{\alpha}
\newcommand{\be}{\beta}
\newcommand{\te}{\theta}
\newcommand{\hbe}{\hat{\beta}}
\newcommand{\hte}{\hat{\theta}}
\newcommand{\hga}{\hat{\gamma}}
\newcommand{\hal}{\hat{\alpha}}
\newcommand{\bbe}{\bar{\beta}}
\newcommand{\bte}{\bar{\theta}}
\newcommand{\bal}{\bar{\alpha}}
\newcommand{\mtr}{\mathrm{tr}}
\newcommand{\mer}{L}

\newcommand{\hE}{\hat{E}}

\author[1]{Luyuan Yang}
\author[1]{Brayden S. Garner}
\author[1]{Shayan Shafaei}
\author[1]{Chao Lan \thanks{Email: clan@ou.edu}}
\affil[1]{School of Computer Science, 
University of Oklahoma}
\date{\today}

\begin{document}

\title{Distributed Sketching on 
Data Partitions for OLS Regression}

\maketitle

\begin{abstract}
This paper studies a scalable approach for 
ordinary least squares (OLS) regression 
based on distributed sketching. 
It distributes small sketches of 
a large data set across multiple workers, 
has each worker construct a local OLS 
estimator from its sketch, and averages 
the local estimators into a global one. 
Prior studies assume each sketch is generated 
from all data, but this sketching 
step has high computation cost, sometimes 
even higher than standard OLS.  
Instead, we propose to generate each 
sketch from a partitioned subset of 
the data, which allows the sketching cost 
to decrease with more workers. Under the 
fixed design setting, we characterize the exact 
excess loss of the global estimator. This 
loss is higher than the loss established in 
prior studies, but remains close to it 
when the average Burg divergence between 
subset covariances is small. Since more 
workers tend to increase the divergence 
under our partition design, this reveals 
a tradeoff between the computational and 
statistical efficiency of our estimator. 
\end{abstract}

\section{Introduction}

An effective approach to scale 
up ordinary least squares (OLS) 
regression is distributed sketching \cite{bartan2023distributed,garg2024distributed,chen2025gpu}. Its basic idea is to distribute 
small sketches of a large data set across 
multiple workers, 
have each worker construct a local OLS 
estimator from its sketch, and average 
the local estimators into a global one. 
An \textit{exact} excess loss of the 
global estimator is characterized 
in \cite{bartan2023distributed}. 

A limitation of the existing approach 
is that it assumes each sketch is 
generated from the entire data set, 
which is computationally expensive 
-- sometimes most costly than standard 
OLS.  
In this paper, we propose an alternative 
sketching mechanism. It generates each 
sketch from a partitioned subset of 
the data and allows the sketching cost 
to decrease as the number of workers increases. 
Under the fixed design setting, we 
characterize the exact excess loss 
of the resulting global estimator. 
This loss is higher than the loss 
established in prior study, but remains 
close to it when the average Burg 
divergence between subset covariances 
is small. Since more workers tend to 
increase the divergence under our 
partition design, this reveals a 
tradeoff between the computational 
efficiency and statistical efficiency 
of the new estimator, balanced through 
the worker number. 

\section{Preliminaries}

Consider the fixed design setting. 
Let $x_1, \ldots, x_n \in \R^d$ 
be a fixed set of $n$ points. 
Let $y_i \in \R$ be the label of $x_i$ 
generated by $y_i = x_i^T \al_* 
+ \e_i$, where $\al_* \in \R^d$ 
is fixed and $\e_i$'s are 
i.i.d. random noises with 
zero mean and a fixed variance 
$\sigma^2 > 0$. 
Let $X = [x_1, \ldots, x_n]^T \in 
\R^{n \times d}$ be the sample matrix, 
$Y = [y_1, \ldots, y_n]^T \in \R^n$ 
be the label vector, 
and $E = [\e_1, \ldots, \e_n]^T \in \R^n$ 
be the noise vector. 
The fixed design setting can 
be written in a matrix form as 
\begin{equation}
\label{eq:settingX}
Y = X \al_* + E.    
\end{equation}

Consider an even 
row partition of $X$ into $k$ 
submatrices $X_1, \ldots, X_k \in 
\R^{p \times d}$, and assume 
$pk = n$ for convenience. Let 
$Y_1, \ldots, Y_k \in \R^p$ 
and $E_1, \ldots, E_k \in \R^p$ 
be the corresponding row 
partitions of $Y$ and $E$, 
respectively. It follows 
from (\ref{eq:settingX}) 
that $Y_i = X_i \al_* + E_i$.

Let $\tS = \{ \tilde{S}_1, \ldots, \tilde{S}_k 
\in \R^{m \times n}\}$ and $S = \{ S_1, 
\ldots, S_k \in \R^{m \times p}\}$
be two sets of independent random matrices 
with i.i.d. entries following $N(0, 1/m)$. 
Let $||\cdot||$ denote $L_2$ norm. 

Consider three OLS estimators. 
Given $E$, the first estimator 
is built on the entire data set i.e., 
\begin{equation}
\hal_* = \arg\min_{\al \in R^d} 
||X \al - Y||^2. 
\end{equation}
Given $(E,\tilde{S})$, 
the second estimator is built 
on sketches mapped from the entire set i.e.,
\begin{equation}
\bbe = \frac{1}{k} \sum_{i=1}^k \hbe_i,
\quad \text{where}\ \ \hbe_i = 
\arg\min_{\al \in \R^d} ||\tS_i 
X \al - \tS_i Y||^2.
\end{equation}
Given $(E,S)$, the third estimator 
is built on sketches mapped from 
partitioned subsets i.e., 
\begin{equation}
\bte = \frac{1}{k} 
\sum_{i=1}^k \hte_i,\quad 
\text{where}\ \  
\hte_i = \arg\min_{\al \in \R^d} 
||S_i X_i \al - S_i Y_i||^2.
\end{equation}

The exact excess loss for $\bbe$ 
has been studied. 
We will characterize the exact loss 
for $\bte$ and make comparisons.
Similar to prior studies, we focus 
on the case $p > d+1$, and 
assume with probability one that  
all $X_i$'s and $X_i^T X_i$'s are 
full rank. To facilitate discussion, 
for any $\al \in \R^d$, define 
\begin{equation}
L(\al) := ||X \al - Y||^2.
\end{equation}

\subsection{Subset Divergence Measure}

A quantity that turns out useful 
in the characterization is 
\begin{equation}
D = \frac{1}{k^2} \sum_{i=1}^k 
\sum_{j=1}^k D_{ij}, \quad 
\text{where}\ \ 
D_{ij} =  \mtr[(X_i^TX_i)(X_j^TX_j)^{-1}].
\end{equation}
This quantity can be viewed as a subset 
divergence measure, 
and a smaller value generally implies 
smaller divergence. The following 
lemma reveals its connection to feature 
dimension. 

\begin{lemma}
\label{lem:Dlower}
With probability one, $D \geq d$ 
where the equality is achieved if  
and only if either $k = 1$ 
(no partition) 
or $X_i^TX_i = X_j^T X_j$ 
for all $i \neq j$ (no divergence). 
\end{lemma}

The quantity has a tight connection 
to the Burg matrix divergence 
\citep{kulis2009low}, also known 
as the Stein's loss in an earlier 
literature \citep{dey1985estimation}. 
To be specific, if each $X_i$ is centered, 
then $X_i^TX_i$ is the 
unscaled sample covariance for subset 
$i$, and the standard Burg divergence 
between subsets $i$ and $j$ is 
\begin{equation}
BD(i,j) = D_{ij} - d - 
\log\det[(X_i^TX_i)(X_j^TX_j)^{-1}].    
\end{equation}
The average Burg divergence across 
all subsets is then 
\begin{equation}
\label{eq:meanBD}
BD = \frac{1}{k^2} 
\sum_{i,j=1}^k BD(i,j) 
= D -d - \frac{1}{k^2} 
\sum_{i,j=1}^k \log \det 
[(X_i^TX_i) (X_j^TX_j)^{-1}] 
= D - d,    
\end{equation}
where the last equality follows 
from $\sum_{i,j}^k \log \det 
[(X_i^TX_i) (X_j^TX_j)^{-1}] 
= \log \prod_{i,j}^k 
\frac{\det (X_i^TX_i)}{
\det (X_j^TX_j)} = \log 1$. 

Equation (\ref{eq:meanBD}) 
shows that larger $D$ implies 
larger Burg divergence. 
Reversely, when the divergence 
is zero -- if and only if each 
$BD(i,j) = 0$ because $BD(i,j) \geq 0$, 
and this condition holds 
if and only if $X_i^TX_i = X_j^T X_j$ 
(a property of Burg divergence) 
-- $D$ attains its minimum 
by Lemma \ref{lem:Dlower}. 

Quantity $D$ also connects to the popular 
leverage score \citep{bharadwaj2023fast}. 
Let $A^r$ denote the 
$r_{th}$ row of a given matrix $A$. 
For subset $i$, the standard leverage 
score of its $r_{th}$ point is
\begin{equation}
L_{ir} = (X^r_{i})\  
(X_i^T X_i)^{-1} (X^r_{i})^T,    
\end{equation}
which specifies the degree to which point 
$r$ deviates from other points in subset $i$. 
Similarly, 
\begin{equation}
D_{ij} = \sum_{r=1}^p L_{ir,j},\quad 
\text{where}\ \ L_{ir,j} = (X^r_{i})\  
(X_j^T X_j)^{-1} (X^r_{i})^T.     
\end{equation}
Based on this, $L_{ir,j}$ can 
be interpreted as a relative leverage 
score that specifies the degree to 
which point $r$ deviates from the points 
in subset $j$. Then, $D_{ij}$  
is the total relative leverage 
score for subset $i$ w.r.t. subset $j$,   
and $D$ is the average relative score 
that describes subset deviations.

\section{Main Results}
\label{sec:mainresult}

We first revisit a known result for $\bbe$. 
\begin{theorem}[\citep{bartan2023distributed}] 
\label{thm:prior}
Given $E$, for $m > d+1$: 
\begin{equation}
\E_{\tS \mid E} [\mer(\bbe)] 
- \mer(\hal_*) = \frac{1}{k} \frac{\mer(\hal_*)}{m - d - 1} d.   
\end{equation}
\end{theorem}
\begin{remark}
\label{rem:prior}
Under the fixed design setting, 
Theorem \ref{thm:prior} implies    
\begin{equation}
\label{thm:priorimplied}
\E_{\tS, E} [\mer(\bbe)] 
- \E_{E} [\mer(\hal_*)]
=  \frac{\sigma^2}{k} 
\frac{n-d}{m - d - 1} d \quad =:\ B_{\be}.
\end{equation}
\end{remark}

Our result for $\bte$ is as follows.

\begin{theorem}
\label{thm:main}
Under the fixed design setting, 
for $p \geq m > d+1$: 
\begin{equation}
\E_{S, E} [\mer(\bte)] 
- \E_{E} [\mer(\hal_*)] = 
\frac{\sigma^2}{k} 
\frac{n-kd}{m-d-1} D
+ \sigma^2 (D - d) \quad =:\ B_{\te}.  
\end{equation}
\end{theorem}

The theorem has three implications. 
\begin{itemize}
\item[(i)] If $k = 1$, then $B_\te = B_{\be}$. 
This makes sense since no subset partition 
is performed under $k = 1$, in which case 
the two estimators are equivalent. 
\item[(ii)] If $D$ is close to $d$, 
then $B_\te$ is close to $B_\be$ 
and may even be smaller. 
This means when subset divergence 
is small, both estimators perform 
similarly and ours may perform better. 
\item[(iii)] If $D$ is far from $d$, then 
$B_\te$ is bigger than $B_\be$ and the 
gap generally increases as $D$ increases. 
This means 
when subset divergence is large, our 
estimator performs worse and the 
performance gap increases as more 
workers participate.  
\end{itemize}

The following toy example gives 
a concrete idea on the gap between 
$D$ and $d$. 

\begin{remark}
\label{claim:ED}
If $x_i$'s are sampled i.i.d. from 
$N(\vec{0}, \V)$ with an invertible 
$\V$, then $\frac{\E_X[D]}{d} 
= \frac{n- (d +1)}{n-k(d+1)}$. 
\end{remark}

In this example, the gap is small 
(in expectation) when $k$ is small or 
$n \gg \max(k,d)$; in particular, 
$\E[D] = d$ if $k = 1$, which 
coincides with the above observation 
(i). Another interesting experiment is 
to plug the above relation into $B_{\te} 
- B_{\be}$, which gives 
\begin{equation}
\label{eq:gap}
\E_X [B_{\te} - B_{\be}] = 
\sigma^2 d \cdot (F(k; n, d, m) - 1),  
\end{equation}
where $F(k; n, d, m) = 
\frac{n - kd}{m-d-1} 
\frac{n-(d+1)}{n-k(d+1)} \frac{1}{k} 
+ \frac{n-(d+1)}{n-k(d+1)} 
- \frac{n-d}{m-d-1} \frac{1}{k}$. 
Through 
analysis, we can show that $F \geq 1$ 
and $F$ is an increasing function of $k$.  
This means our bound is always worse, 
and the gap to prior bound increases 
as more workers participate. We should 
clarify that these behaviors also depend 
on the distribution of $X$ and our partition 
design (so that the subset size is $n/k$). 

Finally, we present a high 
probability bound for the excess 
loss. 

\begin{theorem}
\label{rem:ours2b}
For $m > (\sqrt{d} + \sqrt{2 
\ln(2k/\delta)})^2$ and $\delta > 0$, we have 
\begin{equation}
\label{eq:ours2}
\E_{E \mid S} [L(\bte)] 
- \E_{E}[L(\hal_*)] 
\leq \sigma^2 D \cdot (\Gamma(p,m,k,\delta) 
- 1 ) + \sigma^2 (D - d) \quad =:\ B_{\te}', 
\end{equation}
with probability at least $1 - \delta$ 
over the random choice of $S$,  
where $\Gamma(p,m,k,\delta) 
= \frac{(\sqrt{m} + \sqrt{p} 
+ \sqrt{2\log(2k/\delta)} )^2}{
(\sqrt{m}-\sqrt{d} - \sqrt{2 
\log(2k/\delta)} )^{2}}$. 
\end{theorem}

Comparing $B_{\te}'$ with $B_{\te}$, 
we see they have the same bias term 
$\sigma^2 (D - d)$. 
For $B_{\te}$, the first term 
is $\sigma^2 D \frac{p-d}{m-d-1}$, 
which is dominated by $\sigma^2 D  
\frac{p}{m}$ when $m \gg d$. 
For $B_{\te}'$, by ignoring the logarithmic 
terms, the first term is roughly 
$\sigma^2 D 
( \frac{(\sqrt{m} + \sqrt{p})^2}{ 
(\sqrt{m} - \sqrt{d})^2} - 1 )$, 
which is also dominated by $\sigma^2 D 
\frac{p}{m}$ when $p \gg m \gg d$. 
This suggests that $B_{\te}'$ can be 
a high probability proxy of $B_{\te}$.

\section{Proof Sketches}

\subsection{Theorem \ref{thm:main}}

Given $E$, let the squared loss of an 
estimator $\al$ on $(X_i,Y_i)$ be 
\begin{equation}
\mer_i(\al) = \|X_i \al - Y_i\|^2.   
\end{equation}
Let $\bal$ be an average estimator 
defined as 
\begin{equation}
\bal = \frac{1}{k} \sum_{i=1}^k \hal_{i}, 
\quad \text{where}\ \ 
\hal_{i} = \arg\min_{\al} L_i(\al). 
\end{equation}
We begin with an observation that 
$\E_{S \mid E}[ \hte_i] = \hal_i$ 
(Claim \ref{claim1}), which 
implies that (Claim \ref{eq:claimexp*}) 
\begin{equation}
\label{eq:basicevaluatL}
\E_{E}\left[ \mer(\hal_*) \right] 
= (n-d) \sigma^2 
\qquad\ \text{and} \qquad 
\E_{E} \left[ L_{i}(\hal_{i}) \right]  
= (p-d) \sigma^2.
\end{equation}
The former implies both Remark \ref{rem:prior} 
and an important characterization that 
(Claim \ref{lem2}) 
\begin{equation}
\label{eq:roadmap_00v}
\E_{S, E} \left[ \mer(\bte) \right]
- \E_{E} \left[ \mer(\hal_*) \right] 
= \E_{S, E} \| X (\bte - \hal_*)\|^2.  
\end{equation}
To evaluate the right side, we 
first fix $E$ and show that 
(Claim \ref{claim:decomp_01}) 
\begin{equation}
\label{eq:roadmap_01v}
\E_{S \mid E} \|X \bte - X \hal_*\|^2 
= \frac{1}{k^2} \sum_{i=1}^k 
\E \|X (\hte_i - \hal_i)\|^2 
+ \|X \bal - X \hal_*\|^2, 
\end{equation}
where 
\begin{equation}
\label{eq:proofroadmap_01}
\E_{S \mid E} \| X (\hte_i - \hal_i)\|^2 
= \E \| X_i (\hte_i - \hal_i)\|^2 + \E 
{\sum}_{j \neq i}  \| X_j (\hte_i - \hal_i)\|^2
 = \frac{d+\Delta_i}{m-d-1} \mer_i(\hal_i), 
\end{equation}
where $\Delta_i = \sum_{j \neq i} \mtr 
(X_j^TX_j) (X_i^TX_i)^{-1}$. 
In (\ref{eq:proofroadmap_01}), the 
first equality holds since  
$\|Xz\|^2 = \sum_{i=1}^k 
\|X_i z\|^2$ for all $z \in \R^d$, 
and the second holds because 
$\E \| X_i (\hte_i - \hal_i)\|^2 
= \frac{d}{m-d-1} \mer_i(\hal_i)$ 
by Claim \ref{claim2}\ref{claim2:lemma1}  
and $\E \sum_{j \neq i} \| X_j 
(\hte_i - \hal_i)\|^2 = \frac{\Delta_i}{m-d-1} 
\mer_i(\hal_i)$ by Claim \ref{claim3}. 
Plugging (\ref{eq:proofroadmap_01})
into (\ref{eq:roadmap_01v}) and 
taking expectation w.r.t. $E$ on both 
sides, we have 
\begin{equation}
\label{eq:proofrm_final}
\E_{S, E} \|X \bte - X \hal_*\|^2 
= \frac{1}{k^2} \sum_{i=1}^k 
\frac{d+\Delta_i}{m-d-1} \E_{E} [\mer_i(\hal_i)]
+ \E_{E} \|X \bal - X \hal_*\|^2. 
\end{equation}
On the right side, plug in 
$\E_{E} [L_{i}(\hal_{i})] 
= (p-d) \sigma^2$ from (\ref{eq:basicevaluatL}), 
$\frac{1}{k^2} \sum_{i=1}^k (d+\Delta_i) = D$ 
from Claim \ref{claim:Dd}, and 
$\E_{E} \|X \bal - X \hal_*\|^2 
= (D - d) \sigma^2$ from Claim \ref{claim:error_p3}. This proves the theorem.

\subsection{Theorem \ref{rem:ours2b}}

Given an estimator $\al$, 
define its expected out-of-sample 
excess loss as 
\begin{equation}
L_o(\hal) = \E_{Y'} ||X \hal - Y'||^2, 
\end{equation}
where $Y'$ is an independent copy 
of $Y$ -- the latter is used to 
construct $\hal$. 

We first observe that
for any estimator of the 
form $\hal = A Y$, where 
$A \in \R^{d \times n}$ 
is independent to $Y$ and 
satisfies $A X = I_d$, there 
is (Claim \ref{claim:remark_bias_01}) 
\begin{equation}
\E_{E \mid A}[L_o(\hal)] 
- \E_{E \mid A} [L(\hal)] 
= 2 d \sigma^2.
\end{equation}

Based on it, we show that both 
estimators $\hal_*$ and $\bte$ satisfy 
the above form with different choices 
of $A$ (Claim \ref{claim:thm2_c2}), 
which implies 
$\E_{E \mid S}[L_o(\bte)] 
- \E_{E} [L(\bte)] 
= \E_{E \mid S}[L_o(\hal)] 
- \E_{E} [L(\hal)] 
= 2 d \sigma^2$ 
and thus 
\begin{equation}
\label{eq:hpb_01}
\E_{E \mid S} [L(\bte)]  - \E_E [L(\hal)] = 
\E_{E \mid S}[L_o(\bte)] - \E_{E}[L_o(\hal)].
\end{equation}

We then focus on developing a high probability 
bound for the right side of (\ref{eq:hpb_01}), 
which is easier since label 
noise is independent to the 
estimator. The main strategy 
is to bound the concentration 
of $(X_i^TS_i^T S_i X_i)^{-1}$ 
(Claim \ref{rem:support})  
and plug it into the loss. This 
leads to Claim \ref{thm:main2}, 
which states that for $m > (\sqrt{d} + \sqrt{2 
\ln(2k/\delta)})^2$ and $\delta > 0$, 
we have 
\begin{equation}
\label{eq:ours2_002}
\E_{E \mid S} [L_o(\bte)] 
- \E_{E}[L_o(\hal_*)] 
\leq \sigma^2 D \cdot (\Gamma(p,m,k,\delta) 
- 1 ) + \sigma^2 (D - d), 
\end{equation}
with probability at least $1 - \delta$ 
over the random choice of $S$. Combining 
(\ref{eq:hpb_01}, \ref{eq:ours2_002}) 
proves the theorem.

\section{Numerical Results}

We evaluate the performance of $\bte$ 
on three real-world data sets, which 
are directly loaded from the Python 
scikit-learn datasets library, including 
the \textit{Digit}, \textit{California 
Housing (House)} and \textit{Covertype (Cover)}. 
The Cover data set was down-sampled to 
100k points for efficient testing. 
We preprocessed each data set by 
removing constant features and 
features that are linearly dependent 
on others (for more numerical stability). 
Statistics of the preprocessed 
data sets are shown in Table \ref{tab:statistics}. 

\begin{table}[h!]
\centering
\setstretch{1}
\begin{tabular}{lcc}
\toprule
\textbf{Dataset} & \#\textbf{Instance} 
& \#\textbf{Feature} \\ \midrule
Digit & 1797 & 61 \\
House & 20640 & 8 \\
Cover & 100000 & 52 \\
\bottomrule
\end{tabular}
\caption{Statistics of Data Sets}
\label{tab:statistics}
\end{table}

We focus on evaluating the impact of 
$k$ on the expected excess loss 
estimated from data. 
Other hyper-parameters were 
set as follows: $m = d + 2$, 
all sketching matrices have entries 
i.i.d. sampled from $N(0, 1/m)$; 
to build $\bte$, each data set was 
evenly partitioned based on the 
default data order, and the  
extra $n - pk$ points were abandoned. 
OLS regression model was built 
without bias term. 

Figure \ref{fig:error} shows the 
normalized excess loss versus $k$, 
which were 
averaged over 1000 random and independent 
choices of the sketching matrices.\footnote{Source code 
at  https://github.com/petelan/Distributed-Sketching-for-OLS-Regression.git} 
On Digit, our estimator performs slightly better, 
which matches the second implication of 
Theorem \ref{thm:main}. On House and Cover, 
our estimator performs worse and the 
gap between losses increases 
as $k$ increases, which matches the 
third implication of the main theorem 
and, in particular, implications 
of the toy example based on 
Remark \ref{claim:ED}.

\begin{figure}[t!]
\centering
\includegraphics[width=.32\linewidth]{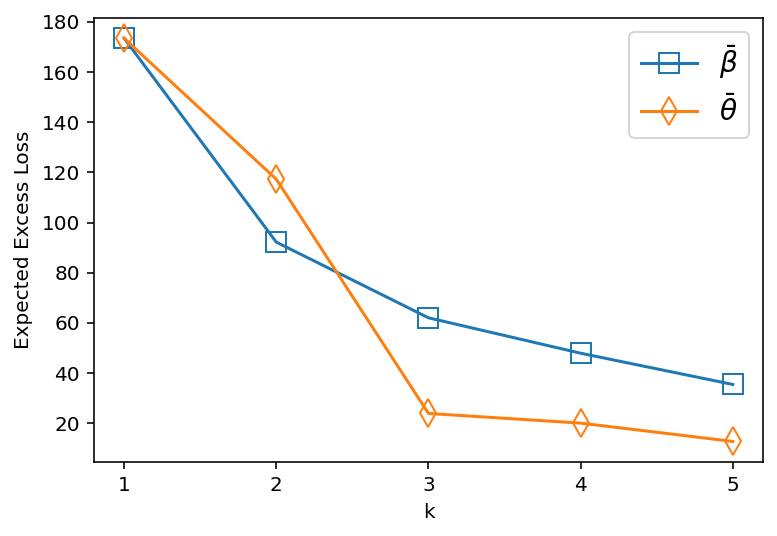}
\includegraphics[width=.32\linewidth]{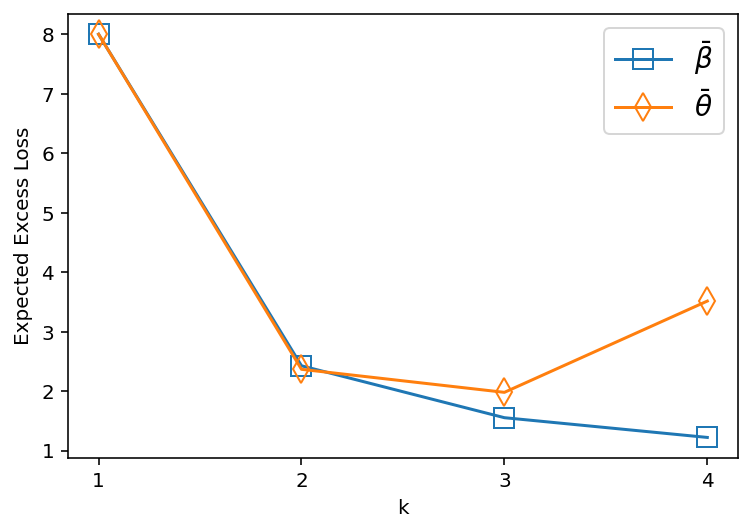}
\includegraphics[width=.32\linewidth]{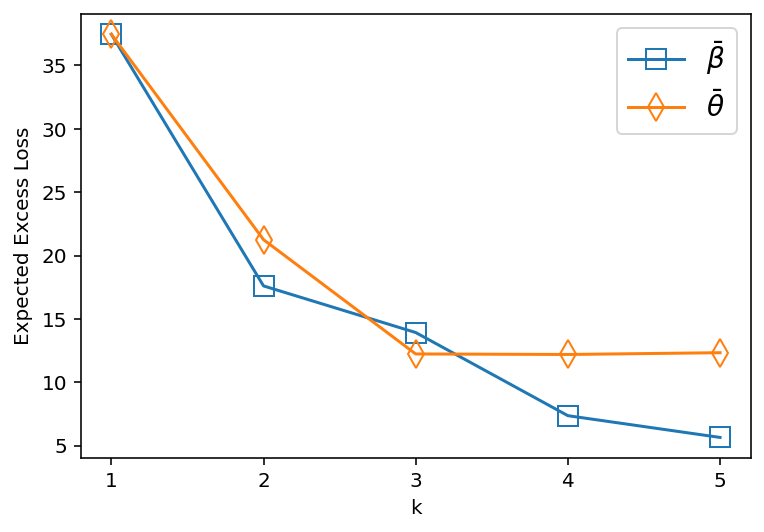}
\vspace{-10pt}
\caption{Normalized Excess Loss on 
Digit (left), House (mid) and Cover (right)}
\label{fig:error}
\end{figure}

\begin{figure}[t!]
\centering
\includegraphics[width=.32\linewidth]{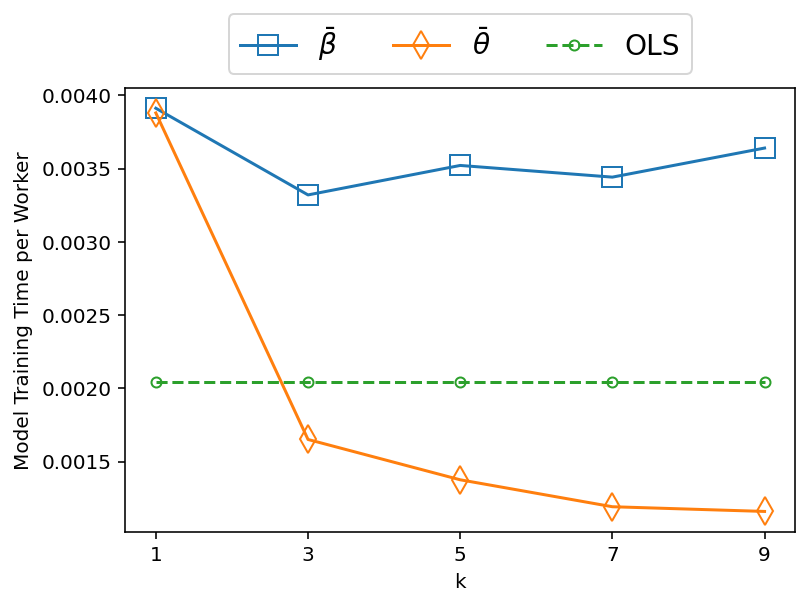}
\includegraphics[width=.32\linewidth]{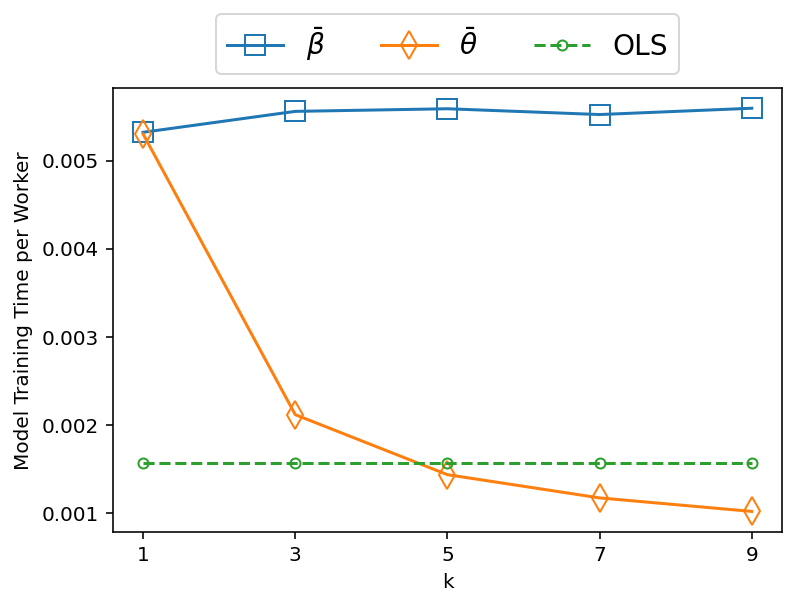}
\includegraphics[width=.32\linewidth]{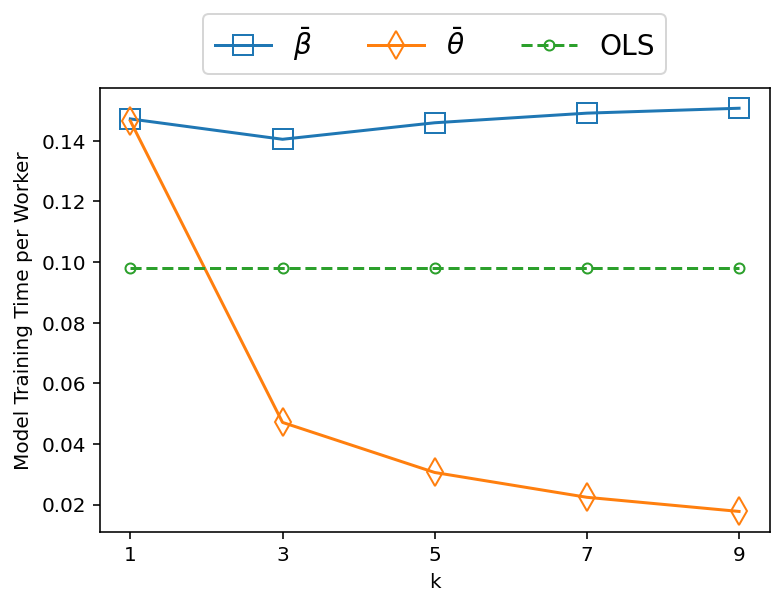}
\vspace{-10pt}
\caption{Model Training Time per worker 
on Digit (left), House (mid) and 
Cover (right)}
\label{fig:time}
\end{figure}

Figure \ref{fig:time} shows the 
time for building a local OLS estimator, 
including the generation 
and application of sketching matrices. 
Here, `OLS' is the approach that directly 
builds a global estimator on $(X,Y)$. 
We see $\hbe_i$ consumes more 
time than OLS, and $\hte_i$ 
reduces the time as $k$ increases. 
These observations match the the following 
time complexity analysis, presuming 
$n > m > d$: 
\begin{itemize}
\item OLS takes $O(n d^2)$ time.
\item $\hbe_i$ applies Gaussian sketching 
on the entire data set, which takes 
$O(mdn)$ time. 
\item $\hte_i$ applies Gaussian sketching 
on a partitioned subset, which takes 
$O(mdn/k)$ time.
\end{itemize}
For large $n$ and small $d$, the sketching 
process of $\hbe_i$ could indeed take more 
time than OLS, while the sketching process 
of $\hte_i$ consumes less time as $k$ 
increases. 

At last, we compare the 
computational efficiency of our estimator 
with Subsampled Randomized Hadamard 
Transform (SRHT). SRHT is an efficient 
sketching method and often used to as 
a faster alternative of Gaussian sketching. 
However, SRHT is not always faster, 
for its time complexity is $O(n d \log n)$ 
which means it is only faster when $m \ll n$. 
This is confirmed by our results 
in Figure \ref{fig:time2}, where $n$ 
is reduced by random downsampling 
and the model training time for SRHT is obtained 
by using it to replace Gaussian sketching 
when building $\hbe_i$.  
We see that SRHT is only faster when $n$ 
is small, while $\hte_i$ remains the most 
efficient method with large $k$.

\begin{figure}[t!]
\centering
\includegraphics[width=.32\linewidth]{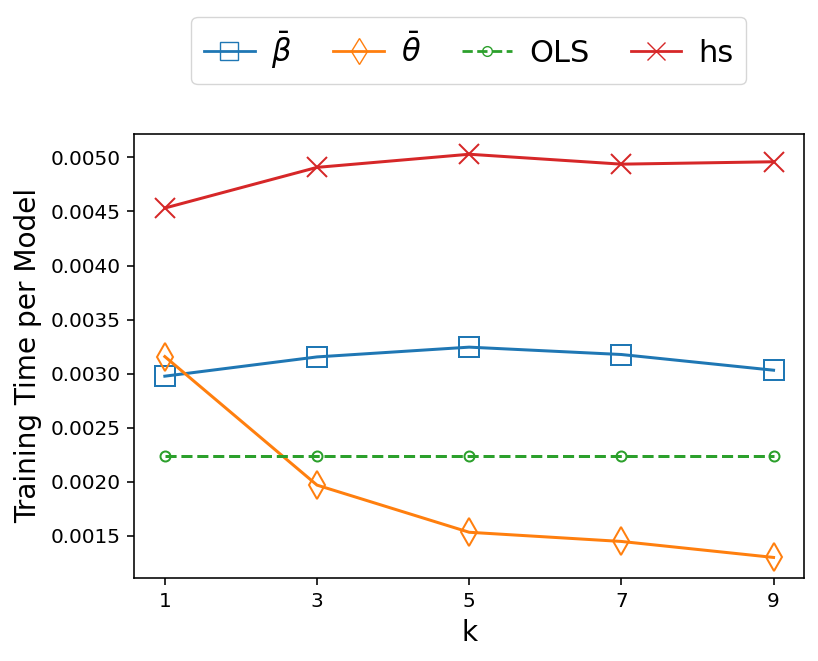}
\includegraphics[width=.32\linewidth]{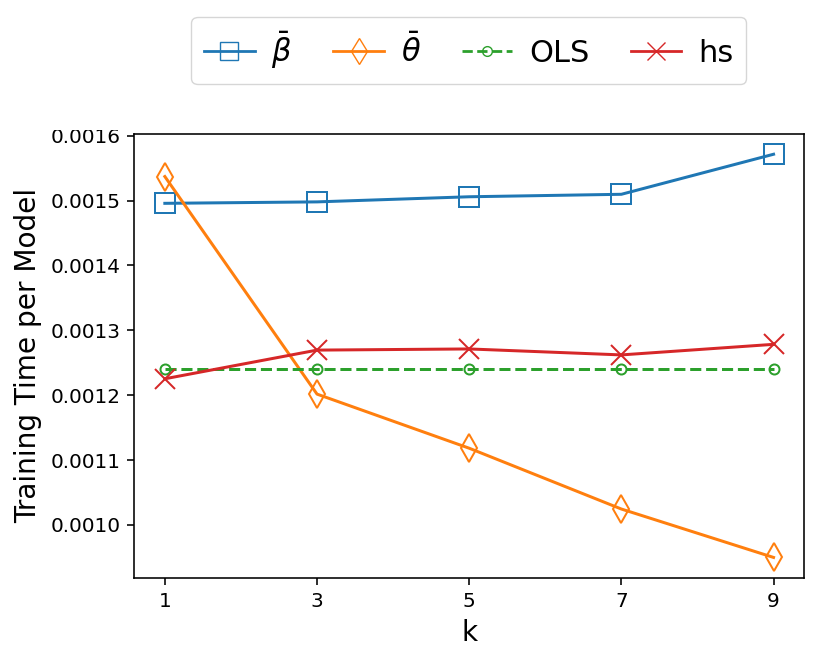}
\includegraphics[width=.32\linewidth]{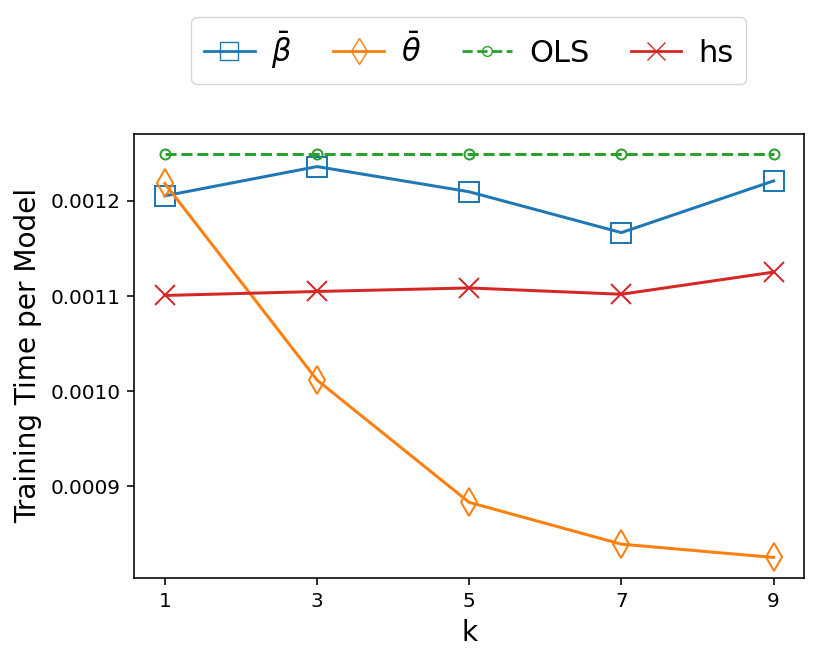}
\vspace{-10pt}
\caption{Training Time per Estimator on 
Digit: $n=1024$ (left), $n=256$ 
(mid), $n=128$ (right).}
\label{fig:time2}
\end{figure}

\section{Acknowledgments}

We thank the anonymous reviewer 
for the feedback that improves the 
quality of the work. Theorem \ref{rem:ours2b} 
is co-developed with Claude. 
Function analysis for 
(\ref{eq:gap}) and (\ref{eq:ours2}) 
are conducted by Claude. 

\bibliographystyle{plain}
\bibliography{reference}

\clearpage 
\appendix
\section{Supporting Claims for 
Theorem \ref{thm:main}}

Recall $\mer_i(\al) = \| X_i \al - Y_i \|^2$ 
and $\hal_{i} = \arg\min_{\al} L_i(\al)$. 
Define residual
\begin{equation}
\label{eq:htei*}
\hE_{i} = Y_i - X_i \hal_{i}.  
\end{equation}
Note that $\hE_i$ is different from 
$E_i$ (which is a subset of label noise 
$E$) and is fixed given $E_i$. 

Recall $S = \{S_1, \ldots, S_k\}$ 
and define $SX = \{S_1 X_1, \ldots, S_k X_k\}$. 
Within each set, elements are mutually 
independent and thus any variable $G_i$ that only 
depends on $S_i$ satisfies $\E_{S} [G_i] 
= \E_{S_i} [G_i]$ and $\E_{SX} [G_i] 
= \E_{S_iX_i} [G_i]$. We also remark 
that $\E_{S_i} [S_i^T S_i] = I_p$, where 
$I_q$ denotes an identity matrix of size 
$q$. In the following analysis, any expectation 
$\E$ without subscript inherits the subscript 
of its preceding expectation.

\begin{claim} 
\label{claim1}
Given $E$, the following hold: 
\begin{enumerate}[label=(\alph*)]
\vspace{-1pt}
\item $X_i^T \hE_{i} = 0$. \label{claim1:orthEX}

\vspace{-3pt}
\item $S_i X_i$ and $S_i \hE_{i}$ are independent. 
\label{claim1:indpEX}

\vspace{-3pt}
\item $\E_{S \mid SX, E} [\hte_i] = 
\E_{S \mid E} [\hte_i] = \hal_{i}$  
and $\E_{S \mid E} [\bte] = \bal$.
\label{claim1:unbhte}
\end{enumerate}
\end{claim}

\begin{proof}
Claim \ref{claim1:orthEX} 
is proved using the fact that 
$\hal_i$ is a minimizer of $L_i(\al)$. 
Therefore 
\begin{equation}
0 = \left.\frac{\partial L_i}{\partial \al} 
\right|_{\al = \hal_i}
= 2 X_i^T X_i \hal_i 
- 2 X_i^T Y_i  = 
2 X_i^T X_i \hal_i 
- 2 X_i^T (X_i \hal_i + \hE_i)
= - 2 X_i^T \hE_i. 
\end{equation}

Claim \ref{claim1:indpEX} follows 
from a basic property that 
jointly-Gaussian variables are 
independent if they are uncorrelated. 
Specifically, with $E$ fixed, 
$S_i X_i$ and $S_i \hE_i$ 
are jointly Gaussian because each 
of them is a linear combination of 
independent Gaussian entries in $S_i$ 
(and so is their linear combination). 
In addition, with Claim 
\ref{claim1:orthEX} and that 
 $\E_{S_i} [S_i^T S_i] = I_p$, 
 we can show the correlation satisfies 
\begin{equation}
corr(S_i X_i, S_i \hE_i \mid E, X) 
\ \propto\ \E_{S \mid E} 
[X_i^T S_i^T S_i \hE_i] = 
X_i^T \hE_i = 0.     
\end{equation}

Claim \ref{claim1:unbhte} follows 
from that $\hte_i = \arg\min_{\al} 
\|S_i X_i \al - S_i Y_i\|^2$ 
and $Y_i = X_i \hal_i + \hE_i$. 
Specifically, 
\begin{equation}
\label{eq:claim_expectation}
\begin{split}
\hte_i & = (X_i^T S_i^T S_i X_i)^{-1} 
X_i^T S_i^T S_i (Y_i)\\[.1em] 
& = (X_i^T S_i^T S_i X_i)^{-1} 
X_i^T S_i^T S_i (X_i \hal_i + \hE_i)\\[.1em] 
& = \hal_i + (X_i^T S_i^T S_i X_i)^{-1} 
X_i^T S_i^T S_i \hE_i.  
\end{split}
\end{equation}
Taking expectation on both sides 
w.r.t. $S$, conditioned on $(SX, E)$, 
we have 
\begin{equation}
\label{eq:claim1_prf_01av}
\E_{S \mid SX, E} [\hte_i] = 
\hal_i + (X_i^T S_i^T S_i X_i)^{-1} 
X_i^T S_i^T\, 
\E_{S \mid SX, E}[S_i \hE_i] 
= \hal_i,     
\end{equation}
where the second equality follows 
from the fact that 
\begin{equation}
\label{eq:claim01_prf_01}
\E_{S \mid SX, E}[S_i \hE_i]
= \E_{S \mid E}[S_i \hE_i]
= \E_{S \mid E}[S_i]\, \hE_i 
= \vec{0}.
\end{equation}
In (\ref{eq:claim01_prf_01}), 
the first equality holds 
because $S_i \hE_i$ and $S_i X_i$ 
are independent 
(Claim \ref{claim1:indpEX}), 
and the second equality holds 
because $\hE_i$ is fixed given 
$E$. Based on the above, we have 
\begin{equation}
\label{eq:claim01_prf_02}
\E_{S \mid E}[\hte_i] 
= E_{SX \mid E} ( \E_{S \mid SX, E}[\hte_i]) 
= E_{SX \mid E}[\hal_i] = \hal_i,  
\end{equation}
where the last equality holds because 
$\hal_i$ does not depend on $S$. 
This proves the first part of 
\ref{claim1:unbhte}. The second 
part of \ref{claim1:unbhte}
follows trivially by averaging 
both sides of (\ref{eq:claim01_prf_02}) 
through $i = 1, \ldots, k$. 
\end{proof}

\begin{claim}
\label{eq:claimexp*}
$\E_{E} [\mer(\hal_*)] 
= (n-d) \sigma^2$\quad  
and\quad  
$\E_{E}[\mer_i(\hal_i)] 
= (p-d) \sigma^2$.    
\end{claim}
\begin{proof} 
Let $G = X(X^TX)^{-1}X^T - I_n$. 
The first part of the claim 
follows from this argument: 
\begin{equation}
\label{eq:app_prf01}
\E_{E} 
[\mer(\hal_*)] = \E \| X \hal_* - Y\|^2 
= \E \|G E\|^2 = \E \mtr[G EE^T G^T]
= \sigma^2 \mtr[GG^T] = \sigma^2 (n-d).  
\end{equation}
In (\ref{eq:app_prf01}), the 
second equality holds because 
$\hal_*$ minimizes $\|X\al-Y\|^2$ 
and $Y = X \al_* + E$, thus 
\begin{equation}
\begin{split}
\|X \hal_* - Y\|^2 
& = \|X (X^TX)^{-1} X^T Y - Y\|^2\\
& = \|X (X^TX)^{-1} X^T (X \al_* + E) 
- (X \al_* + E)\|^2\\
& = \|X \al_* + X (X^TX)^{-1} X^T E  
- X \al_* - E \|^2\\
& = \|X (X^TX)^{-1} X^T E - E \|^2\\
& = \|G E\|^2.     
\end{split}
\end{equation}
In (\ref{eq:app_prf01}), the 
last equality follows from a 
straight calculation that 
\begin{equation}
\begin{split}
\mtr[GG^T] & = \mtr[ X(X^TX)^{-1}X^T 
X(X^TX)^{-1}X^T] - 2 
\mtr[X(X^TX)^{-1}X^T] + \mtr[I_n]\\[.1em] 
& = \mtr[I_d] - 2 \mtr[I_d] + \mtr[I_n]\\[.1em] 
& = n-d.        
\end{split}
\end{equation}
The second part of the claim follows 
from the same argument, 
with $(X,Y, \hal_*)$ replaced 
by $(X_i,Y_i, \hal_i)$. 
\end{proof}

\begin{claim}
\label{lem2}
$\E_{S, E} [\mer(\bte)] - 
\E_{E} [\mer(\hal_*)] = 
\E_{S, E} \|X (\bte - \hal_*)\|^2$.
\end{claim}
\begin{proof} 
We begin with an observation that 
\begin{equation}
\label{claim:charloss_01}
\begin{split}
\E_{S, E} \|X (\bte - \hal_*)\|^2 
& = \E \|X \bte - Y + Y - X \hal_*\|^2\\ 
& = \E \|X \bte - Y\|^2 + 
\E \|X \hal_* - Y\|^2 
- 2 \E (X \hal_* - Y)^T (X \bte - Y)\\ 
& = \E [\mer(\bte)] + 
\E [\mer(\hal_*)] - 
2 \E (X \hal_* - Y)^T (X \bte - Y). 
\end{split}
\end{equation}    
Now we show that the last 
term equals $\E [\mer(\hal_*)]$. 
First,  
\begin{equation}
\label{eq:sup_prf02}
\begin{split}
\E_{S, E} (X \hal_* - Y)^T (X \bte - Y) 
& = \E_{E} \left[ \E_{S \mid E} 
(X \hal_* - Y)^T (X \bte - Y) \right] \\[.1em]
& = \E_{E} (X \hal_* - Y)^T (X 
\E_{S \mid E} [\bte] - Y) \\[.1em]
& = \E_{E} 
(X \hal_* - Y)^T (X \bal - Y) \\[.1em]
& = \E \hal_*^T X^T  X \bal + 
\E Y^T Y - \E Y^T X \bal - \E \hal_*^T 
X^T Y\\[.1em]
& = \E Y^T Y - \E \hal_*^T X^T Y,      
\end{split}
\end{equation}
where the third equality follows  
from Claim \ref{claim1}\ref{claim1:unbhte} 
and the fact that only $\bte$ depends 
on $S$, and the last equality holds because 
$\hal_*^T X^T X \bal 
= Y^T X (X^TX)^{-1} X^T X 
\bal= Y^T X \bal$. 

Now we aim to show the right 
side of (\ref{eq:sup_prf02}) 
equals $\E[L(\hal_*)]$. 
Since $Y = X \al_* + E$, 
\begin{equation}
\label{eq:term1}
\begin{split}
\E_{E} [Y Y^T] 
& = \E (X \al_* + E) (X \al_* + E)^T \\[.1em] 
& = \E [ X \al_* \al_*^T X^T + E E^T 
+ X \al_* E^T + E \al_*^T X^T ] \\[.1em] 
& = X \al_* \al_*^T X^T + \E [E E^T] \\[.1em] 
& = X \al_* \al_*^T X^T 
+ \sigma^2 I_n,         
\end{split}
\end{equation}
where the third equality holds 
because $E$ has zero mean, and the 
last equality holds because $E$ has 
i.i.d. entries with zero mean and 
variance $\sigma^2$. 
Based on the above, we have 
\begin{equation}
\label{eq:evalEYY}
\E_{E} [Y^T Y]
= \E \mtr[Y Y^T]
= \mtr[\E (Y Y^T)]
= \mtr[X \al_* \al_*^T X^T 
+ \sigma^2 I_n] 
= \al_*^T X^T X \al_* + n \sigma^2. 
\end{equation}
To evaluate the second term on 
the right side of (\ref{eq:sup_prf02}), 
recall that $\hal_*$ 
minimizes $\|X \al - Y\|^2$. Then, 
\begin{equation}
\label{eq:term2}
\begin{split}
\E_{E} [\hal_*^T X^T Y] 
& = \E Y^TX (X^TX)^{-1} X^T Y\\  
& = \E \mtr[X (X^TX)^{-1} X^T YY^T] \\
& = \mtr[X (X^TX)^{-1} X^T \E(YY^T)] \\
& = \mtr[X (X^TX)^{-1} X^T 
(X \al_* \al_*^T X^T + \sigma^2 I_n)] \\
& = \mtr[X \al_* \al_*^T X^T 
+ \sigma^2 X (X^TX)^{-1} X^T)] \\
& = \al_*^T X^T X \al_*  
+ \sigma^2 \mtr [I_d] \\
& = \al_*^T X^T X \al_*  + d \sigma^2,  
\end{split}
\end{equation}
where the fifth equality follows 
from (\ref{eq:evalEYY}). 
Plugging (\ref{eq:term1}, \ref{eq:term2}) 
back to (\ref{eq:sup_prf02}) yields 
\begin{equation}
\E_{S, E} (X \hal_* - Y)^T 
(X \bte - Y) = \E Y^T Y - \E \hal_*^T X^T Y 
= (n-d) \sigma^2 = \E_{E \mid X} 
[\mer(\hal_*)], 
\end{equation}
where the last equality follows 
from Claim \ref{eq:claimexp*}. 
Plugging the above back to (\ref{claim:charloss_01}), we have 
\begin{equation}
\begin{split}
\E_{S, E \mid X} \|X (\bte - \hal_*)\|^2 
& = \E L(\bte) + \E L(\hal_*)
- 2 \E (X \hal_* - Y)^T (X \bte - Y)\\[.1em] 
&  = \E L(\bte) + \E L(\hal_*) - 
2 \E L(\hal_*) \\[.1em] 
& = \E L(\bte) - \E L(\hal_*).     
\end{split}
\end{equation}    
This proves the lemma. 
\end{proof}

\begin{claim}
\label{claim:decomp_01}
$\E_{S \mid E} \|X \bte - X \hal_*\|^2 
= \frac{1}{k^2} \sum_{i=1}^k 
\E \|X (\hte_i - \hal_i)\|^2 
+ \|X \bal - X \hal_*\|^2$.
\end{claim}
\begin{proof}
Since neither $\bal$ or $\hal_*$ 
depends on $S$, we have 
\begin{equation}
\label{eq:claim44_b}
\begin{split}
\E_{S \mid E} \|X \bte - X \hal_*\|^2
& = \E \|X \bte - X \bal + X \bal 
- X \hal_*\|^2\\
& = \E \|X \bte - X \bal\|^2 
+ \|X \bal - X \hal_*\|^2 + 
2 \E (\bte - \bal)^T X^T 
X (\bal - \hal_*)\\
& = \E \|X \bte - X \bal\|^2 
+ \|X \bal - X \hal_*\|^2 + 
2 (\E [\bte] - \bal)^T X^T 
X (\bal - \hal_*)\\
& = \E \|X \bte - X \bal\|^2 
+ \|X \bal - X \hal_*\|^2,  
\end{split}
\end{equation}    
where the third equality 
holds because only $\bte$ depends 
on $S$, and the last holds because 
$\E_{S \mid E} [\bte] = \bal$ by Claim \ref{claim1}\ref{claim1:unbhte}. 
Next, we separately evaluate the two terms 
on the right side of (\ref{eq:claim44_b}). 

For the first term, by the designs 
of averaged estimators, 
\begin{equation}
\label{eq:claim44_a}
\begin{split}
\E_{S \mid E} \|X \bte - X \bal\|^2 
& = \E \| \frac{1}{k}\sum_{i=1}^k 
X (\hte_i - \hal_i)\|^2 \\
& = \E \frac{1}{k^2} \sum_{i=1}^k 
\| X (\hte_i - \hal_i)\|^2 + \E 
\frac{1}{k^2} \sum_{i=1}^k \sum_{j \neq i} 
(\hte_i - \hal_i)^T X^T 
X (\hte_j - \hal_{j})\\
& = \E \frac{1}{k^2} \sum_{i=1}^k 
\| X (\hte_i - \hal_i)\|^2.  
\end{split}
\end{equation} 
In (\ref{eq:claim44_a}), the double sum 
equals zero based on two arguments: 
first, $\hte_i, \hte_j$ are independent 
because they are respectively built on 
$S_i, S_j$ which are independent; second, 
Claim \ref{claim1}\ref{claim1:unbhte} 
says $\E_{S \mid E} [\hte_i] = \hal_i$ 
for any $i$; as a result, every term in 
the double sum satisfies 
\begin{equation}
\E_{S \mid E} 
(\hte_i - \hal_i)^T X^T X (\hte_j - \hal_{j})
= ( \E[\hte_i] - \hal_i)^T X^T X 
( \E[\hte_j] - \hal_{j}) 
= ( \hal_i - \hal_i)^T X^T X 
( \hal_j - \hal_{j}) = 0.    
\end{equation}

Plugging (\ref{eq:claim44_a}) 
into (\ref{eq:claim44_b}) proves the claim. 
\end{proof}

\begin{claim}
\label{claim2}
Given $E$, the following hold: 
\begin{enumerate}[label=(\alph*)]
\item $\E_{S \mid E} \| X_i (\hte_i - \hal_i)\|^2 
= \frac{d}{m-d-1} \mer_i(\hal_i)$. \label{claim2:lemma1}
\item $\E_{S \mid SX, E} 
(\hte_i - \hal_i) (\hte_i - \hal_i)^T 
= (X_i^T S_i^T S_i X_i)^{-1} 
\frac{1}{m} \mer_i(\hal_i)$. 
\label{claim2:htecovariance}
\item $\E_{S} (X_i^T S_i^T S_i X_i)^{-1} 
= (X_i^T X_i)^{-1} \frac{m}{m-d-1}$. 
\label{claim2:expinv}
\end{enumerate}
\end{claim}
\begin{proof}
All results are directly adapted 
from \citep[Lemma 1]{bartan2023distributed}.  
This lemma says given any $(X,Y)$  
and Gaussian sketching $Q \in \R^{m \times n}$ 
satisfying $\E Q^TQ = I_n$, we have 
$\E_{Q} \|X (\hga_Q - \hga)\|^2 
= \frac{d}{m-d-1} L(\hga)$, 
where $\hga$ is a minimizer 
of $L(\al) = \|X \al - Y\|^2$ and 
$\hga_Q$ is a minimizer of 
$\|Q X \al - Q Y\|^2$. 
The original proof invokes two results. 
First, $\E_{Q \mid QX} 
(\hga_Q - \hga)(\hga_Q - \hga)^T 
= (X^TQ^TQX)^{-1} \frac{1}{m} 
 L(\hga)$ based on straight 
 calculations and the fact that 
 $\E_{Q \mid Q X} [\hga_Q] = \hga$. 
 Second, 
 $\E_{Q} (X^T Q^T Q X)^{-1} 
 = (X^T X)^{-1} \frac{m}{m-d-1}$ 
 since $X^T Q^T Q X$ is a 
 Wishart matrix whose expected 
 inverse is well-known. 

We simply adapt the above results 
from $(Q X, QY)$ to $(Q X_i, Q Y_i)$. 
This is done by replacing the above 
quantities $(X,Y,Q,\hga,\hga_Q, L)$ 
with $(X_i,Y_i,S_i,\hal_i,\hte_i, L_i)$, 
plus $\E [\hte_i] = \hal_i$ in Claim \ref{claim1}\ref{claim1:unbhte}. 
\end{proof}

\begin{claim}
\label{claim3}
For any $i$, we have 
\begin{equation}
\E_{S \mid E} \sum_{j \neq i} 
\| X_j (\hte_i - \hal_i)\|^2 
= \frac{\Delta_i}{m-d-1} \mer_i(\hal_i),
\end{equation}
where $\Delta_i = \mtr [ \sum_{j \neq i} 
X_j^T X_j (X_i^T X_i)^{-1} ]$ and all 
sums are over $j = 1, \ldots, k$ 
with $j = i$ skipped. 
\end{claim}
\begin{proof}
First observe that 
\begin{equation}
\begin{split}
\E_{S \mid SX, E} 
\| X_j (\hte_i - \hal_i)\|^2 
& = \E \mtr [  X_j (\hte_i 
- \hal_i) (\hte_i - \hal_i)^T X_j^T]\\
& = \mtr [  X_j^T X_j\, \E (\hte_i 
- \hal_i) (\hte_i - \hal_i)^T]\\
& = \mtr [ X_j^T X_j \frac{\mer_i(\hal_i)}{m}
(X_i^T S_i^T S_i X_i)^{-1} ]        
\end{split}
\end{equation}
where the last equality follows from 
Claim \ref{claim2}\ref{claim2:htecovariance}. 
Taking expectation of the 
right side w.r.t. $SX$, 
conditioned on $E$, we have 
\begin{equation}
\begin{split}
\E_{SX \mid E} 
\mtr [ X_j^T X_j \frac{\mer_i(\hal_i)}{m} 
\cdot (X_i^T S_i^T S_i X_i)^{-1}]
& = \mtr [ X_j^T X_j \frac{\mer_i(\hal_i)}{m} 
\cdot \E (X_i^T S_i^T S_i X_i)^{-1}]\\
& = \mtr [  X_j^T X_j \frac{\mer_i(\hal_i)}{m}
\cdot \frac{m\, (X_i^T X_i)^{-1}}{m-d-1}] \\ 
& = \mtr[(X_j^T X_j)(X_i^T X_i)^{-1}] 
\frac{\mer_i(\hal_i)}{m-d-1},      
\end{split}
\end{equation}   
where the second equality holds 
by Claim \ref{claim2}\ref{claim2:expinv}. 
Summing both sides over  
$j$ yields the claim. 
\end{proof}

\newpage 
\begin{claim}
\label{claim:Dd}
$\sum_{i=1}^k (d+\Delta_i) = k^2 D$.     
\end{claim}
\begin{proof}
The key observation is that 
\begin{equation}
d = \mtr[I_d] = 
\mtr[X_i^T X_i (X_i^T X_i)^{-1}].     
\end{equation}
Therefore, 
\begin{equation}
d + \Delta_i  
= \mtr[X_i^T X_i (X_i^T X_i)^{-1}] 
+ \mtr [ \sum_{j \neq i} 
X_j^T X_j (X_i^T X_i)^{-1} ]   
= \mtr [ \sum_{j=1}^k X_j^T X_j 
(X_i^T X_i)^{-1} ],  
\end{equation}
and thus 
\begin{equation}
 \sum_{i=1}^k (d+\Delta_i) 
 = \sum_{i=1}^k \mtr [ 
 \sum_{j=1}^k X_j^T X_j 
 (X_i^T X_i)^{-1} ] 
 = \mtr [ \sum_{i=1}^k \sum_{j=1}^k 
 X_j^T X_j (X_i^T X_i)^{-1} ] 
 = k^2 D,   
\end{equation}
where the last equality follows from the 
definition of $D$. This proves the claim. 
\end{proof}

\begin{claim}
\label{claim:error_p3}
$\E_{E} \|X \bal - X \hal_*\|^2 
= (D - d) \sigma^2$.
\end{claim}

\begin{proof}
Since $Y = X \al_* + E$ 
and $\hal_*$ is a minimizer 
of $\|X \al - Y\|^2$, 
\begin{equation}
\hal_* 
= (X^T X)^{-1} X^T Y
= (X^T X)^{-1} X^T (X \al_* + E)
= \al_* + (X^T X)^{-1} X^T E.     
\end{equation}
Similarly, since $Y_i = X_i \al_* + E_i$ 
and $\hal_i$ is a minimizer of 
$\|X_i \al - Y_i\|^2$, 
\begin{equation}
\hal_i 
= (X_i^T X_i)^{-1} X_i^T Y_i
= (X_i^T X_i)^{-1} X_i^T (X_i \al_* + E_i)
= \al_* + (X_i^T X_i)^{-1} X_i^T E_i.    
\end{equation}
Define 
\begin{equation}
H_i = \frac{1}{k} (X_i^T X_i)^{-1} 
X_i^T E_i.    
\end{equation}
The above implies 
\begin{equation}
\bal = \frac{1}{k} \sum_{i=1}^k \hal_i 
= \frac{1}{k} \sum_{i=1}^k \al_* 
+ \frac{1}{k} \sum_{i=1}^k 
(X_i^T X_i)^{-1} X_i^T E_i 
= \al_* + \sum_{i=1}^k H_i.      
\end{equation}
This further implies 
$\bal - \hal_* = \sum_{i=1}^k H_i 
- (X^T X)^{-1} X^T E$ 
and thus  
\begin{equation}
\label{eq:sup_prf04}
\begin{split}
& \E_{E} \|X (\bal - \hal_*)\|^2\\ 
& = \E \|X \sum_{i=1}^k H_i 
- X (X^T X)^{-1} X^T E\|^2\\ 
& = \E \|X \sum_{i=1}^k H_i\|^2 
+ \E \|X (X^T X)^{-1} X^T E\|^2
- 2 \E (X \sum_{i=1}^k H_i)^T 
X (X^T X)^{-1} X^T E. 
\end{split}
\end{equation}
On the right side, we will evaluate 
the three expectations separately. 

\vspace{5pt} \noindent 
\textbf{1st Term}: To evaluate 
$\E_{E} \|X \sum_{i=1}^k H_i\|^2$, 
observe that 
\begin{equation}
\label{eq:prof_claimlast_01}
\begin{split}
\E_{E} \|X \sum_{i=1}^k H_i\|^2 
& = \E \mtr [X \sum_{i,j=1}^k H_i H_j^T X^T] \\ 
& = \E \mtr [X \sum_{i=1}^k H_i H_i^T X^T  
+ X \sum_{i=1}^k \sum_{j \neq i} 
H_i H_j^T X^T]\\
& = \E \mtr [X \sum_{i=1}^k H_i H_i^T X^T] 
+ \E \mtr [ X \sum_{i=1}^k \sum_{j \neq i} 
H_i H_j^T X^T]\\
& =  \mtr [X \sum_{i=1}^k \E(H_i H_i^T) X^T] 
+ \mtr [ X \sum_{i=1}^k \sum_{j \neq i} 
\E (H_i H_j^T) X^T]\\
& = \mtr [X \sum_{i=1}^k \E(H_i H_i^T) X^T].   
\end{split}
\end{equation}
where the last equality holds because  
$E$ has i.i.d. zero-mean entries 
so that $\E_{E} E_i E_j^T = 0 I_{p}$ 
and thus $
\E_{E} H_i H_j^T 
= \E \frac{1}{k^2} (X_i^T X_i)^{-1} 
X_i^T E_i E_j^T X_j (X_j^T X_j)^{-1}
= \frac{1}{k^2} (X_i^T X_i)^{-1} 
X_i^T \E (E_i E_j^T) X_j (X_j^T X_j)^{-1} 
= 0 I_d$. 

It remains to evaluate the right 
side of (\ref{eq:prof_claimlast_01}). 
Since $\E_{E} E_i E_i^T = \sigma^2 I_p$, 
\begin{equation}
\label{eq:prof_claimlast_02}
\begin{split}
\E_{E} H_i H_i^T 
& = \E \frac{1}{k^2} 
(X_i^T X_i)^{-1} X_i^T E_i 
E_i^T X_i (X_i^T X_i)^{-1}\\  
& = \frac{1}{k^2} 
(X_i^T X_i)^{-1} X_i^T \E (E_i 
E_i^T) X_i (X_i^T X_i)^{-1} \\
& = \frac{\sigma^2}{k^2} 
(X_i^T X_i)^{-1}.    
\end{split}
\end{equation}
Plugging (\ref{eq:prof_claimlast_02}) 
into (\ref{eq:prof_claimlast_01}),
plus the fact that $X^TX = \sum_{j=1}^k 
(X_j^T X_j)$, we have 
\begin{equation}
\label{eq:sup_prf07}
\begin{split}
\E_{E} \|X \sum_{i=1}^k H_i\|^2
& = \mtr [X \sum_{i=1}^k 
\E(H_i H_i^T) X^T ]\\
& = \mtr [ \sum_{i=1}^k \E (H_i H_i^T) 
X^T X ]\\
& = \mtr [\frac{\sigma^2}{k^2} 
\sum_{i=1}^k (X_i^T X_i)^{-1} 
\sum_{j=1}^k (X_j^T X_j) ]\\ 
& = \frac{\sigma^2}{k^2} 
\mtr [\sum_{i,j=1}^k (X_i^T X_i)^{-1} 
(X_j^T X_j)] \\
& = \sigma^2 D.
\end{split}
\end{equation}
\textbf{2nd Term}: To evaluate 
$\E \|X (X^T X)^{-1} X^T E\|^2$, 
recall $\E_{E} E E^T = \sigma^2 I_n$. 
Then 
\begin{equation}
\label{eq:sup_prf08}
\begin{split}
\E_{E} \|X (X^T X)^{-1} X^T E\|^2 
& = \E \mtr [ X (X^T X)^{-1} X^T 
 E E^T X (X^T X)^{-1} X^T ] \\[.1em]
& = \mtr [ X (X^T X)^{-1} X^T 
 \E (E E^T) X (X^T X)^{-1} X^T ] \\[.1em]
 & = \sigma^2 \mtr [ X (X^T X)^{-1} X^T 
 X (X^T X)^{-1} X^T ] \\[.1em]
 & = \sigma^2 \mtr [ I_d] = d \sigma^2.
\end{split}
\end{equation}
\textbf{3rd Term}: To evaluate 
$\E_{E} (X \sum_{i=1}^k H_i)^T 
X (X^T X)^{-1} X^T E$, 
recall $X^T E = \sum_{j=1}^k X_j^T E_j$. 
Then 
\begin{equation}
\label{eq:sup_prf08b}
\begin{split}
\E_{E} (X \sum_{i=1}^k H_i)^T 
X (X^T X)^{-1} X^T E 
& = \E \sum_{i=1}^k H_i^T X^T E\\ 
& = \E \sum_{i=1}^k \sum_{j=1}^k 
H_i^T X_j^T E_j \\
& = \E \sum_{i=1}^k H_i^T X_i^T E_i 
+ \E \sum_{i=1}^k\sum_{j\neq i} 
H_i^T X_j^T E_j \\
& = \E \sum_{i=1}^k H_i^T X_i^T E_i, 
\end{split}
\end{equation}
where the last equality holds 
because $E_j$ and $H_i$ (depending on $E_i$) 
are independent for $j \neq i$ so
\begin{equation}
\E_{E} \sum_{i=1}^k\sum_{j\neq i} 
H_i^T X_j^T E_j = 
\sum_{i=1}^k\sum_{j\neq i} 
\E [H_i]^T X_j^T \E[E_j] = \vec{0}.
\end{equation}
It remains to evaluate the right side 
of (\ref{eq:sup_prf08b}). 
By plugging in $H_i = \frac{1}{k} (X_i^T X_i)^{-1} 
X_i^T E_i$, it is 
\begin{equation}
\begin{split}
\E_{E} \sum_{i=1}^k H_i^T X_i^T E_i
& = \E \frac{1}{k} \sum_{i=1}^k E_i^T 
X_i (X_i^T X_i)^{-1} X_i^T E_i \\ 
& = \E \frac{1}{k} \sum_{i=1}^k \mtr [ X_i (X_i^T 
X_i)^{-1} X_i^T E_i E_i^T]\\ 
& = \frac{1}{k} \sum_{i=1}^k \mtr [ X_i (X_i^T 
X_i)^{-1} X_i^T \E (E_i E_i^T)]\\ 
& = \frac{\sigma^2}{k} 
\sum_{i=1}^k \mtr [ X_i (X_i^T 
X_i)^{-1} X_i^T]\\ 
& = \frac{\sigma^2}{k} 
\sum_{i=1}^k \mtr [ I_d] = d \sigma^2. 
\end{split}
\end{equation}
Plugging this back into 
(\ref{eq:sup_prf08b}), the 3rd term is 
\begin{equation}
\E_{E} (X \sum_{i=1}^k H_i)^T 
X (X^T X)^{-1} X^T E =  d \sigma^2.    
\end{equation}
\noindent \textbf{Wrap Up}. 
Plugging the evaluations 
of all terms back into 
(\ref{eq:sup_prf04}), we have 
\begin{equation}
\begin{split}
& \E_{E} \|X (\bal - \hal_*)\|^2\\ 
& = \E \|X \sum_{i=1}^k H_i\|^2 
+ \E \|X (X^T X)^{-1} X^T E\|^2 
- 2 \E (X \sum_{i=1}^k H_i)^T 
X (X^T X)^{-1} X^T E\\
& = D \sigma^2 + d \sigma^2 - 2 
d \sigma^2\\
& = (D - d) \sigma^2.
\end{split}
\end{equation}
This proves the claim. 
\end{proof}

\section{Proof of Lemma \ref{lem:Dlower}}
\label{sec:appB}

\begin{lemma*}
With probability one, $D \geq d$ 
where the equality is achieved if 
and only if $X_i^TX_i = X_j^T X_j$ 
for all $i \neq j$.  
\end{lemma*}
\begin{proof}
First recall our basic assumption that 
subset sample covariance matrices are 
full rank with probability one.
Consider two cases. 

(L1) When $j = i$, straight calculation 
shows that 
\begin{equation}
D_{ij} =  \mtr[(X_i^TX_i)(X_j^TX_j)^{-1}] 
= \mtr [I_d] = d.
\end{equation}

(L2) When $j \neq i$, we will analyze 
$D_{ij} + D_{ji}$ for each pair 
of $(i,j)$. Let 
\begin{equation}
G_{ij} = (X_i^TX_i)(X_j^TX_j)^{-1},     
\end{equation}
and $\lambda_t(\cdot)$ be the $t_{th}$ 
eigenvalue of a matrix (sorted in 
non-increasing order). 

We first argue that, with probability 
one, $G_{ij}$ has full rank 
because both $X_i^TX_i$ and 
$X_j^TX_j$ have full rank. We 
also remark that all eigenvalues 
of $G_{ij}$ are positive because 
they equal to the top $d$ 
eigenvalues of $X_i(X_j^TX_j)^{-1} X_i^T$ 
which is a positive semidefinite 
matrix of rank $d$. Then 
\begin{equation}
\label{eq:appB_Dlowerbound_01}
\begin{split}
D_{ij} + D_{ji} 
= \mtr[G_{ij}] + \mtr[G_{ji}]
& = \sum_{t=1}^d \lambda_t(G_{ij}) 
+ \sum_{t=1}^d \lambda_t(G_{ji}) \\
& = \sum_{t=1}^d \lambda_t(G_{ij}) 
+ \lambda_t(G_{ji}) \\
& = \sum_{t=1}^d \lambda_t(G_{ij}) 
+ \frac{1}{\lambda_t(G_{ij})}\\
& \geq \sum_{t=1}^d 2 \sqrt{\lambda_t(G_{ij}) 
\frac{1}{\lambda_t(G_{ij})}}\\
& = \sum_{t=1}^d 2 \\
& = 2 d. 
\end{split}
\end{equation}
In (\ref{eq:appB_Dlowerbound_01}), 
the third line holds because 
$G_{ij}$ and $G_{ji}$ 
have reciprocal eigenvalues 
due to $G_{ij} = G_{ji}^{-1}$,  
and the fourth line follows 
from the AM-GM inequality 
(with the equality attained 
if and only if $\lambda_t(G_{ij}) = 1$ 
for every $t$). 
Combining cases (L1) and (L2), we have 
\begin{equation}
\label{eq:appB_Dlowerbound_03}
\begin{split}
D = \frac{1}{k^2} \sum_{i,j=1}^k D_{ij} 
= \frac{1}{k^2} \sum_{i=1}^k D_{ii} 
+ \frac{1}{k^2} \sum_{i=1}^k \sum_{j \neq i} D_{ij} 
\geq \frac{1}{k^2} k d + 
\frac{1}{k^2} \frac{k(k-1)}{2} 2d
= d,     
\end{split}
\end{equation}
where the equality is attained 
if and only if $\lambda_t(G_{ij}) = 1$ 
for all $t, i, j$ -- Note this condition 
means all eigenvalues of $G_{ij}$ 
are one's, and the following remark 
helps us better understand its implication. 
\begin{remark}
If all eigenvalues of $G_{ij}$ 
are 1's, then $G_{ij}$ is identity. 
\end{remark}
\begin{proof}
Since $X_i^TX_i$ and $X_j^T X_j$ 
are positive semidefinite,  
$G_{ij} = (X_i^TX_i)(X_j^TX_j)^{-1}$ 
is diagonalizable  
\cite[Corollary 7.6.2]{horn2012matrix}. 
A property of diagonalizable matrix 
is that, if all eigenvalues are one's, 
then the matrix is identity. To see 
why, recall the definition of 
diagonalizability, which means 
$G_{ij} = P \Sigma_G P^{-1}$ 
for some invertible $P$ 
and a diagonal $\Sigma_G$ 
with eigenvalues on the diagonal. 
By presumption $\Sigma_G$ is identity, 
which implies $G_{ij} = P I P^{-1} = I$. 
This proves the remark. 
\end{proof}
Based on this remark, we conclude that 
the equality in (\ref{eq:appB_Dlowerbound_03})
is obtained if and only if every $G_{ij}$ 
is identity, which means $X_i^TX_i = X_j^TX_j$ 
for all $i \neq j$. This proves the lemma. 
\end{proof}

\section{Proof of Remark \ref{claim:ED}}
\begin{remark*}
If $x_i$'s are sampled i.i.d. 
from $N(\vec{0}, \V)$ with an 
invertible $\V \in \R^{d \times d}$, 
then 
\begin{equation}
\E_X[D] = \frac{n-(d+1) 
}{n-k(d+1)} d.    
\end{equation}
\end{remark*}
\begin{proof} 
By assumption $X_j^T X_j$ 
is a Wishart matrix with $p$ degrees 
of freedom. It is well-known that 
$\E_X[X_j^T X_j] = p \V$ and 
$\E_X[(X_j^T X_j)^{-1}] 
= \frac{\V^{-1}}{p-d-1}$ 
\citep[Theorems 7.8.1 and 7.9.1]{petersen2008matrix}. Then 
\begin{equation}
\label{eq:prf_toy_01}
\begin{split}
\E_X[D]  & = \E \frac{1}{k^2} 
\sum_{i,j=1}^k 
\mtr (X_i^TX_i)(X_j^TX_j)^{-1} \\ 
& = \E \frac{1}{k^2} 
\sum_{i=1}^k \mtr (X_i^TX_i)(X_i^TX_i)^{-1} 
+ \E \frac{1}{k^2} 
\sum_{i=1}^k \sum_{j \neq i} 
\mtr (X_i^TX_i)(X_j^TX_j)^{-1}\\
& = \E \frac{1}{k^2} 
\sum_{i=1}^k \mtr[I_d] + \frac{1}{k^2} 
\sum_{i=1}^k \sum_{j \neq i} 
\mtr \left[ \E_{X_i}(X_i^TX_i)\ 
\E_{X_j}(X_j^TX_j)^{-1} \right]\\
& = \frac{d}{k} + \frac{1}{k^2} 
\sum_{i=1}^k \sum_{j \neq i} 
\frac{p}{p-d-1} \mtr[I_d]\\
& = \frac{d}{k} + 
\frac{k-1}{k} \frac{pd}{p-d-1} \\
& = \frac{d}{k} 
\left( 1 + \frac{p(k-1)}{p-d-1} \right)\\
& = \frac{n-d-1}{n-kd-k} d. 
\end{split}
\end{equation} 
In (\ref{eq:prf_toy_01}), the third equality 
holds as $X_i$ and $X_j$ are independent, 
and the last is due to $pk = n$. 
\end{proof}

\section{Supporting Claims for 
Theorem \ref{rem:ours2b}}

\begin{claim}
\label{claim:remark_bias_01}
For any estimator of the 
form $\hal = A Y$, where 
$A \in \R^{d \times n}$ 
is independent to $Y$ and 
satisfies $A X = I_d$, there is 
\begin{equation}
\E_{E \mid A}[L_o(\hal)] 
- \E_{E \mid A} [L(\hal)] 
= 2 d \sigma^2.
\end{equation}
\end{claim}

\begin{proof}
Fix $A$. 
Let $Y' = X \al_* + E'$, where $E'$ 
is an independent copy of $E$. Then, 
\begin{equation}
\label{eq:claimD1_02}
\begin{split}
L_o(\hal) 
& = \E_{Y' \mid A} \| X \hal - Y' \|^2 \\
& = \E  \| X \hal - Y + Y - Y' \|^2\\
& = \| X \hal - Y \|^2 + \E \| Y - Y' \|^2 
+ 2 \E (X \hal - Y)^T (Y - Y')\\ 
& = L(\hal) + \E \| E - E' \|^2 
+ 2 (X \hal - Y)^T (Y - \E Y'), 
\end{split}
\end{equation}    
where the last line holds because 
$L(\hal) =  X \hal - Y ^2$ 
and $Y - Y' = E - E'$ 
by the designs of $Y, Y'$. 
The last two terms on the 
right side of (\ref{eq:claimD1_02}) 
can be evaluated separately. 

For the second term in 
(\ref{eq:claimD1_02}), 
\begin{equation}
\E_{Y' \mid A} \| E - E'\|^2 
= \E_{E' \mid A} \| E - E'\|^2 
= \| E \|^2 + \E \| E'\|^2 
+ 2 \E [E^T E'] =  \|E\|^2 
+ \E \|E'\|^2.
\end{equation}
For the third term in 
(\ref{eq:claimD1_02}), 
since $Y - \E Y' = 
Y - X \al_* = E$, we have  
\begin{equation}
\begin{split}
(X \hal - Y)^T (Y - \E Y') 
& = (X \hal - Y)^T E \\ 
& = (X A Y - Y)^T E \\ 
& = Y^T (XA - I_n)^T E \\
& = (X  \al_* + E)^T (XA - I_n)^T E \\
& = (X  \al_*)^T (XA - I_n)^T E 
+ E^T (XA - I_n)^T E, 
\end{split}
\end{equation}
where the 4th line follows from 
$Y = X  \al_* + E$. Plugging all 
into (\ref{eq:claimD1_02}), 
\begin{equation}
 L_o(\hal) 
= L(\hal) + \| E \|^2 
+ \E_{E' \mid A} \| E'\|^2
+ 2 (X  \al_*)^T (XA - I_n)^T E 
+ 2 E^T (XA - I_n)^T E.    
\end{equation}
Taking expectation w.r.t. 
$E$ on both sides, plus 
$\E_{E' \mid A} \| E'\|^2 = 
\E_{E \mid A} \| E\|^2$, we have 
\begin{equation}
\label{eq:claimD1_03}
\E_{E \mid A} [L_o(\hal)] 
= \E [L(\hal)] + 2 \E \| E \|^2 
+ 2 \E [ E^T (XA - I_n)^T E ].
\end{equation}
The last two terms on the right 
side can be evaluated as follows:  
\begin{equation}
\E_{E \mid A} \| E \|^2 
= \E \mtr [E E^T] 
= n \sigma^2, 
\end{equation}
and 
\begin{equation}
\begin{split}
\E_{E \mid A} [ E^T (XA - I)^T E ] 
& = \E\, \mtr [ E E^T  (XA - I_n)^T ] \\
& = \E [E E^T ] \cdot \mtr (XA - I_n)^T \\ 
& = \sigma^2 \cdot \mtr (XA - I_n)^T\\ 
& = \sigma^2 \cdot [ \mtr (XA) 
- \mtr (I_n)] \\ 
& = \sigma^2 \cdot [ \mtr (AX) 
- \mtr (I_n)] \\ 
& = \sigma^2 \cdot [ \mtr (I_d) 
- \mtr (I_n)] \\ 
& = \sigma^2 \cdot (d - n). 
\end{split}
\end{equation}
Plugging both back into 
(\ref{eq:claimD1_03}), 
\begin{equation}
\E_{E \mid A} [L_o(\hal)] 
= \E [L(\hal)] + 2 n \sigma^2 
+ 2 (d-n) \sigma^2 
= \E [L(\hal)] + 2 d \sigma^2.
\end{equation}
This proves the claim. 
\end{proof}

\begin{claim}
\label{claim:thm2_c2}
Both estimators $\hal_*$ and $\bte$ 
have the form $A Y$ (with different 
choices of $A$), where $A$ is 
independent to $Y$ and satisfies 
$AX = I_d$ deterministically.
\end{claim}
\begin{proof}
Since $\hal_*$ minimizes 
$\| X \al - Y \|^2$, we have 
$\hal_* = (X^T X)^{-1} X^T Y 
= A Y$, where $A = (X^T X)^{-1} X^T$. 
This $A$ is independent to $Y$ 
and satisfies $A X = I_d$. 
This verifies the claim on 
$\hal_*$. 

For $\bte$, since $\hte_i$ minimizes 
$\| S_i X_i \al - S_i Y_i \|^2$, we have 
$\hte_i = (X_i^T S_i^T S_i X_i)^{-1} X_i^T 
S_i^T S_i Y_i$ 
and thus 
\begin{equation}
\bte = \frac{1}{k} \sum_{i=1}^k 
\hte_i = \frac{1}{k} \sum_{i=1}^k 
(X_i^T S_i^T S_i X_i)^{-1} X_i^T 
S_i^T S_i Y_i = \bar{A} Y, 
\end{equation}
where 
\begin{equation}
\bar{A} = \frac{1}{k}[ (X_1^T S_1^T S_1 X_1)^{-1} 
X_1^T S_1^T S_1,\ \ldots,\  (X_k^T S_k^T 
S_k X_k)^{-1} X_k^T S_k^T S_k]    
\end{equation}
and $Y = \begin{bmatrix} 
Y_1\\ \vdots \\ Y_k 
\end{bmatrix}$. It is noted that 
$\bar{A}$ is independent to $Y$. 
Further, 
\begin{equation}
\begin{split}
\bar{A} X 
& = \frac{1}{k}[ (X_1^T S_1^T S_1 X_1)^{-1} 
X_1^T S_1^T S_1,\ \ldots,\  (X_k^T S_k^T 
S_k X_k)^{-1} X_k^T S_k^T S_k] 
\begin{bmatrix} 
X_1\\ \vdots \\ X_k 
\end{bmatrix} \\
& = \frac{1}{k} \sum_{i=1}^k 
(X_i^T S_i^T S_i X_i)^{-1} 
X_i^T S_i^T S_i X_i \\
& = \frac{1}{k} \sum_{i=1}^k I_d\\
& = I_d.
\end{split}
\end{equation}
In summary, $\bte= \bar{A} Y$, 
where $\bar{A}$ is 
independent to $Y$ and satisfies 
$\bar{A} X = I_d$ deterministically. 
\end{proof}

\begin{claim}
\label{rem:support}
Recall $S_i \in \R^{m \times p}$ is a random 
matrix with i.i.d. entries following 
$N(0, 1/m)$, and $X_i \in \R^{p \times d}$ 
is a fixed sample matrix with $p > d$. 
For any $\delta > 0$, each of the following 
two events holds with probability at least 
$1 - \delta/(2k)$ over the random choices 
of $S_i$: 
\begin{enumerate}
\item If $m > (\sqrt{d} + \sqrt{2 
\ln(2k/\delta)})^2$, then 
$(X_i^TS_i^T S_i X_i)^{-1} 
\preceq 
\left(1-\sqrt{\frac{d}{m}} - 
\sqrt{\frac{2 \log(2k/\delta)}{m}} 
\right)^{-2} (X_i^T X_i)^{-1}$.   
\item $\|S_i S_i^T\|_{op} \leq 
\left(1 + \sqrt{\frac{p}{m}} + 
\sqrt{\frac{2\log(2k/\delta)}{m}} 
\right)^2$, 
\end{enumerate}
where $\preceq$ is Loewner order 
and $\|\cdot\|_{op}$ is operator 
norm. 
\end{claim}

\begin{proof}
To prove the first claim, let 
\begin{equation}
X_i = U_i R_i    
\end{equation}
be a QR decomposition 
of $X_i$, where $U_i \in \R^{p 
\times d}$ has orthonormal columns 
and $R_i \in \R^{d \times d}$. 
Our first major claim is as follows. 

\begin{claim*}[D.4.a]
\label{claim:sub_thm2_01}
$S_iU_i$ has i.i.d. entries from $N(0, 1/m)$
\end{claim*}
\begin{proof}
We first remark that $S_iU_i$ has i.i.d. rows, 
simply because $S_i$ has i.i.d. rows and $U_i$ 
is fixed. 
In the following, we prove that each row 
of $S_iU_i$ has i.i.d. entries. 

Let $s \in \R^{1 \times p}$ 
denote an arbitrary row of $S_i$. 
We can show that $s U_i$ is a vector whose 
entries are jointly Gaussian, because 
any linear combination of these entries 
can be expressed as $s U_i c$ for some constant $c \in \R^d$ -- this expression 
is a sum of (scaled) independent Gaussian entries of $s$ and thus is itself a Gaussian. 
Moreover, entries of $s U_i$ are uncorrelated 
because their means are 
\begin{equation}
\E [s U_i] = \E [s] U_i = \vec{0},
\end{equation}
and thus their covariance matrix is 
\begin{equation}
\text{cov}(sU_i) = \E (s U_i - \E [s U_i])^T 
(s U_i - \E [s U_i]) 
= \E (sU_i)^T(sU_i) 
= U_i^T \E [s^T s] U_i 
= \frac{1}{m} U_i^T U_i 
= \frac{1}{m} I_d, 
\end{equation}
where the last equality holds 
because $U_i$ has orthonormal 
columns. This is a diagonal 
covariance matrix, suggesting 
the entries of $s U_i$ are 
uncorrelated. Putting all together, 
entries of $s U_i$ are jointly Gaussian 
and uncorrelated, which implies 
their independence. 

Finally, entries $s U_i$ 
follow identical distribution, 
because each is a linear combination 
of $s$ entries with a unit-norm 
linear coefficient vector (which is one 
column of $U_i$) and thus follows a 
zero-mean Gaussian with variance $1/m$. 
\end{proof}

Claim D.4.a verifies that 
$S_i U_i$ has i.i.d. entries 
from $N(0,1/m)$. Our next step 
is to apply the following 
classical random matrix theory 
to bound its singular value. 

\begin{claim*}[\citep{vershynin2012introduction}, 
Corollary 5.35]
\label{rem:svbound}
Let $Z$ be an $m \times p$ matrix 
whose entries are independent 
standard normal random variables. 
Then for every $t \geq 0$, 
with probability at least $1 - 2 
\exp(-t^2/2)$, 
\begin{equation}
\sqrt{m} - \sqrt{p} - t 
\leq \sigma_{\min}(Z) 
\leq \sigma_{\max}(Z) 
\leq \sqrt{m} + \sqrt{p} + t,  
\end{equation}
where $\sigma_{\max}$ and 
$\sigma_{\min}$ denote the 
maximum and minimum singular 
values, respectively. 
\end{claim*}

Applying 
\citep[Corollary 5.35]{vershynin2012introduction}, 
we know that with probability\footnote{The original 
theorem simultaneously bounds the largest 
and smallest singular values with a failure 
probability $2 \exp(-t^2/2)$. For bounding 
only one of these values, the failure probability 
is $\exp(-t^2/2)$.} at least 
$1 - \exp(-t^2/2)$, 
\begin{equation}
\sigma_{\min}(\sqrt{m} S_i U_i) 
\geq \sqrt{m} - \sqrt{d} - t.
\end{equation}
Multiplying both sides by $1/\sqrt{m}$ and 
setting $\exp(-t^2/2) = \delta / (2k)$, the 
above becomes
\begin{equation}
\label{eq:thm2_prf_04}
\sigma_{\min}(S_i U_i) \geq 
1 - \sqrt{\frac{d}{m}} - 
\sqrt{\frac{2 \ln (2k/\delta)}{m}}.  
\end{equation}
with probability at least $1 - 
\delta / (2k)$. 

If the right side of (\ref{eq:thm2_prf_04}) 
is positive i.e., $1 - \sqrt{\frac{d}{m}} - 
\sqrt{\frac{2 \ln (2k/\delta)}{m}} 
> 0$, then the above implies 
\begin{equation}
(S_i U_i)^T (S_i U_i) 
\succeq 
\left(1 - \sqrt{\frac{d}{m}} 
- \sqrt{\frac{2 \ln (2k/\delta)}{m}} 
\right)^2 I_d, 
\end{equation}
and thus 
\begin{equation}
X_i^TS_i^T S_i X_i
= R_i^T U_i^T S_i^T S_i U_i R_i  
\succeq \left(1 - \sqrt{\frac{d}{m}} 
- \sqrt{\frac{2 \ln (2k/\delta)}{m}} 
\right)^2 R_i^T R_i,  
\end{equation}
and further 
\begin{equation}
\label{eq:thm2_prf_05}
(X_i^TS_i^T S_i X_i)^{-1}
\preceq \left(1 - \sqrt{\frac{d}{m}} 
- \sqrt{\frac{2 \ln (2k/\delta)}{m}} 
\right)^{-2} (R_i^T R_i)^{-1},  
\end{equation}
where $\succeq$ and $\preceq$ are Loewner orders.

Finally, plugging the following 
equivalence into (\ref{eq:thm2_prf_05})
proves the first part of Claim \ref{rem:support}:
\begin{equation}
X_i^T X_i = R_i^T U_i^T U_i R_i = R_i^T R_i,   
\end{equation}
where the last equality holds because 
$U_i$ has orthonormal columns. 
The condition of the first claim 
follows from the above presumption 
that $1 - \sqrt{\frac{d}{m}} - 
\sqrt{\frac{2 \ln (2k/\delta)}{m}} 
> 0$. 

The second part of Claim \ref{rem:support} 
follows directly from \citep[Corollary 5.35]{vershynin2012introduction},
which states that with probability at 
least $ 1 - \exp(-t^2/2)$, 
\begin{equation}
\sigma_{\max}(\sqrt{m} S_i) \leq 
\sqrt{m} + \sqrt{p} + t.
\end{equation}
Multiplying both sides by $1/\sqrt{m}$ and 
setting $\exp(-t^2/2) = \delta / (2k)$, 
we have 
\begin{equation}
\sigma_{\max}(S_i) \leq 
1 + \sqrt{\frac{p}{m}} + \sqrt{
\frac{2 \ln (2k/\delta)}{m}},  
\end{equation}
with probability 
at least $1 - \delta / (2k)$. 
Taking square on both sides 
and applying 
$\|S_i^T S_i\|_{op} = 
\sigma_{\max}^2(S_i)$ 
proves the second part of Claim \ref{rem:support}. 
\end{proof}

\begin{claim}
\label{thm:main2}
For $m > (\sqrt{d} + \sqrt{2 
\ln(2k/\delta)})^2$ and $\delta > 0$, we have 
\begin{equation}
\E_{E \mid S} [L_o(\bte)] 
- \E_{E}[L_o(\hal_*)] 
\leq \sigma^2 D \cdot (\Gamma(p,m,k,\delta) 
- 1 ) + \sigma^2 (D - d) \quad =:\ B_{\te}', 
\end{equation}
with probability at least $1 - \delta$ 
over the random choice of $S$, 
where $\Gamma(p,m,k,\delta) 
= \frac{(\sqrt{m} + \sqrt{p} 
+ \sqrt{2\log(2k/\delta)} )^2}{
(\sqrt{m}-\sqrt{d} - \sqrt{2 
\log(2k/\delta)} )^{2}}$.
\end{claim}

\begin{proof}
Recall that $L_o(\al) = \E_{Y'} \|X \al 
- Y'\|^2$, where $Y'$ and $\hal$ are built 
from independent label noises $E'$ and $E$ respectively. Fix $S$. We will evaluate 
$L_o(\bte)$ and $L_o(\hal_*)$ separately.

\vspace{7.5pt} \noindent 
\underline{Part I. Evaluation of $L_o(\bte)$}

\vspace{5pt} \noindent 
Since $\E_{E'} [E'] = \vec{0}$ and 
$\E_{E'} \| E' \|^2 = \sigma^2 n$
, we have 
\begin{equation}
\label{eq:objremark_ours_00}
\begin{split}
L_o(\bte) & = \E_{Y'} \|X \bte - Y'\|^2\\
& = \E_{E'} \|X \bte - X \al_* - E'\|^2\\
& = \|X \bte - X \al_*\|^2 + \E \|E'\|^2 
- 2 \E (E')^T (X \bte - X \al_*)\\
& = \|X \bte - X \al_*\|^2 + 
\sigma^2 n.
\end{split}
\end{equation}
To evaluate the first term, 
we take expectation w.r.t. $E$ 
and recall that $Y_i = X_i \al_* + E_{i}$. Then, 
\begin{equation}
\label{eq:objremark_ours_01}
\E_{E} \|X \bte - X \al_*\|^2 
= \E \| X \frac{1}{k} 
\sum_{i=1}^k \hte_i - X \al_*\|^2 
= \frac{1}{k^2} \E \| \sum_{i=1}^k 
X (\hte_i - \al_*)\|^2 
= \frac{1}{k^2} \sum_{i=1}^k 
\E \| X (\hte_i - \al_*)\|^2,  
\end{equation}
where the last equality holds because 
the cross terms is zero; to be specific, 
since 
\begin{equation}
\label{eq:modeli}
\begin{split}
\hte_i & = (X_i^T S_i^T S_i X_i)^{-1}    
X_i^T S_i^T S_i (Y_i) \\
& = (X_i^T S_i^T S_i X_i)^{-1}    
X_i^T S_i^T S_i (X_i \al_* + E_i)\\ 
& = \al_* + (X_i^T S_i^T S_i X_i)^{-1} 
X_i^T S_i^T S_i E_i,      
\end{split}
\end{equation}
we have 
\begin{equation}
\E_{E}[\hte_i] = 
\E_{E_i}[\hte_i] = 
\al_* + (X_i^T S_i^T S_i X_i)^{-1} 
X_i^T S_i^T S_i\, \E[\hte_i] = \al_*,  
\end{equation}
and thus the cross term is 
\begin{equation}
\E_{E} \sum_{i=1}^k \sum_{j \neq i} 
(\hte_i - \al_*)^T X^T X (\hte_j - 
\al_*) = \sum_{i=1}^k \sum_{j \neq i} 
(\E_{E_i}[\hte_i] - \al_*)^T X^T X 
(\E_{E_j} [\hte_j] - \al_*) = 0. 
\end{equation}
Next, formula (\ref{eq:modeli}) 
also implies 
\begin{equation}
\hte_i - \al_* = (X_i^T S_i^T S_i X_i)^{-1} 
X_i^T S_i^T S_i E_i.      
\end{equation}
Plugging it into the right side 
of (\ref{eq:objremark_ours_01}), 
we have 
\begin{equation}
\label{eq:rem_ours_02}
\begin{split}
& \frac{1}{k^2} \sum_{i=1}^k 
\E_{E} \| X (\hte_i - \al_*)\|^2 \\
& = \frac{1}{k^2} \sum_{i=1}^k 
\E\ \text{tr} [ X^T X 
(\hte_i - \al_*) (\hte_i - \al_*)^T]\\
& = \frac{1}{k^2} \sum_{i=1}^k 
\text{tr} [ X^T X\ \E (\hte_i - \al_*) 
(\hte_i - \al_*)^T]\\     
& = \frac{1}{k^2} \sum_{i=1}^k 
\text{tr} [ X^T X\ \E (X_i^T S_i^T S_i X_i)^{-1} X_i^T S_i^T S_i\ E_i E_i^T\ 
S_i^T S_i X_i (X_i^T S_i^T S_i X_i)^{-1}]\\    
& = \frac{1}{k^2} \sum_{i=1}^k 
\text{tr} [ X^T X\ (X_i^T S_i^T S_i X_i)^{-1} X_i^T S_i^T S_i\ \E_{E_i}(E_i E_i^T)\ S_i^T S_i 
X_i (X_i^T S_i^T S_i X_i)^{-1}]\\
& = \frac{\sigma^2}{k^2} \sum_{i=1}^k 
\text{tr} [ X^T X (X_i^T S_i^T S_i X_i)^{-1} X_i^T S_i^T S_i S_i^T S_i 
X_i (X_i^T S_i^T S_i X_i)^{-1}]\\
& \leq \frac{\sigma^2}{k^2} \sum_{i=1}^k 
\|S_i S_i^T\|_{op} \cdot 
\text{tr} [ X^T X (X_i^T S_i^T S_i X_i)^{-1} X_i^T S_i^T S_i 
X_i (X_i^T S_i^T S_i X_i)^{-1}]\\
& = \frac{\sigma^2}{k^2} \sum_{i=1}^k 
\|S_i S_i^T\|_{op} \cdot 
\text{tr} [ X^T X (X_i^T S_i^T S_i X_i)^{-1}]. 
\end{split}
\end{equation}
In (\ref{eq:rem_ours_02}), the inequality follows 
from a property that $A^T M A \preceq \|M\|_{op} 
A^T A$ for any symmetric positive semidefinite 
matrix $M$ (which is $S_iS_i^T$ is our case) 
and properly sized matrix $A$. 

To proceed, we apply Claim  
\ref{rem:support} to bound every 
$\|S_i S_i^T\|_{op}$ and $(X_i^T S_i^T S_i X_i)^{-1}$ on the right side. The claim 
bounds each term with probability at least 
$1 - \delta / 2k$. We apply it to simultaneously 
bound all $2k$ terms in the sum, and merge the failure probability using a union bound. As 
a result, we conclude that 
for $m > (\sqrt{d} + \sqrt{2 \ln(2k/\delta)})^2$, 
with probability at least $1 - \delta$, 
\begin{equation}
\label{eq:rem_ours_03}
\begin{split}
& \frac{\sigma^2}{k^2}\sum_{i=1}^k 
\|S_i S_i^T\|_{op} \cdot 
\text{tr} [ X^T X (X_i^T S_i^T S_i X_i)^{-1}] \\
& \leq \frac{\sigma^2}{k^2} \sum_{i=1}^k 
\frac{\left(1 + \sqrt{p/m} + 
\sqrt{2\log(2k/\delta)/m} 
\right)^2}{\left(1-\sqrt{d/m} 
- \sqrt{2 \log(2k/\delta)/m} 
\right)^{2}} \text{tr}[(X^T X) 
(X_i^T X_i)^{-1}]\\
& = \Gamma(p,m,k,\delta) \cdot 
\frac{\sigma^2}{k^2} \sum_{i=1}^k 
\text{tr} [(X^T X) (X_i^T X_i)^{-1}]\\
& = \Gamma(p,m,k,\delta) \cdot \sigma^2 D,  
\end{split}
\end{equation}
where $\Gamma(p,m,k,\delta) = \frac{
\left(1 + \sqrt{p/m} + 
\sqrt{2\log(2k/\delta)/m} 
\right)^2}{\left(1-\sqrt{d/m} 
- \sqrt{2 \log(2k/\delta)/m} 
\right)^{2}}$ and the last equality 
holds because 
\begin{equation}
\begin{split}
\sum_{i=1}^k \text{tr} [(X^T X) 
(X_i^T X_i)^{-1}] 
& = \text{tr} [(X^T X) \sum_{i=1}^k 
(X_i^T X_i)^{-1}] \\
& = \text{tr} [\sum_{j=1}^k (X_j^T X_j) \sum_{i=1}^k (X_i^T X_i)^{-1}] \\
& = \text{tr} [\sum_{i,j=1}^k (X_j^T X_j) 
(X_i^T X_i)^{-1}] \\ & = k^2 D.     
\end{split}
\end{equation}
Plugging (\ref{eq:rem_ours_03}) 
into (\ref{eq:rem_ours_02}) and 
the result into 
(\ref{eq:objremark_ours_01}), we have 
\begin{equation}
\E_{E} \|X \bte - X \al\|^2  
\leq \Gamma(p,m,k,\delta) \cdot \sigma^2 D. 
\end{equation}
Combining this with 
(\ref{eq:objremark_ours_00}) yields 
a high probability bound 
\begin{equation}
\label{eq:rem_ours_p1}
\E_{E} [L_o(\bte)] \leq 
\Gamma(p,m,k,\delta) \cdot 
\sigma^2 D + \sigma^2 n.    
\end{equation}

\vspace{10pt}
\noindent 
\underline{Part II. Evaluation of $L_o(\hal_*)$}

\vspace{5pt} \noindent 
Evaluation of $L_o(\hal_*)$ is similar 
to that of $L_o(\bte)$. First, 
\begin{equation}
L_o(\hal_*) = \E_{Y'} 
\|X \hal_* - Y'\|^2 
= \|X \hal_* - X \al_*\|^2 
+ \sigma^2 n,  
\end{equation}
which implies 
\begin{equation}
\label{eq:rem_ours_part2_01}
\E_{E}[L_o(\hal_*)] 
= \E_{E} \|X \hal_* - X \al_*\|^2 
+ \sigma^2 n. 
\end{equation}
Since 
\begin{equation}
\hal_* 
= (X^TX)^{-1} X^T Y 
=  (X^TX)^{-1} X^T (X \al + E)
= \al + (X^TX)^{-1} X^T E
\end{equation}
we have 
\begin{equation}
\begin{split}
\E_{E} \|X \hal_* - X \al_*\|^2 
& = \E \text{tr} [X (X^TX)^{-1} X^T\  
E E^T\ X (X^TX)^{-1} X^T] \\
& = \text{tr} [X (X^TX)^{-1} X^T\ 
\E (E E^T)\ X (X^TX)^{-1} X^T] \\
& = \sigma^2 \text{tr} [X (X^TX)^{-1} X^T 
X (X^TX)^{-1} X^T] \\
& = \sigma^2 \text{tr} [I_d] \\
& = \sigma^2 d. 
\end{split}
\end{equation}
Plugging this back to (\ref{eq:rem_ours_part2_01}), 
we have 
\begin{equation}
\label{eq:rem_ours_p2}
\E_{E}[L_o(\hal_*)] 
= \sigma^2 d + \sigma^2 n.
\end{equation}

\noindent 
\underline{Part III. Summary}

\vspace{5pt} \noindent 
Combining (\ref{eq:rem_ours_p1}, 
\ref{eq:rem_ours_p2}), we have 
that with probability at least $1 - \delta$, 
\begin{equation}
\E_{E} [L_o(\bte)] 
- \E_{E}[L_o(\hal_*)] 
\leq \Gamma(p,m,k,\delta) \cdot 
D \sigma^2 - d \sigma^2, 
\end{equation}
where $m > (\sqrt{d} + \sqrt{2 
\ln(2k/\delta)})^2$ and 
$\Gamma(p,m,k,\delta)$ is defined 
above. This proves the claim. 
\end{proof}

\end{document}